\documentclass[11pt]{article}

\usepackage{acl}
\usepackage{natbib}
\usepackage{times}
\usepackage{latexsym}

\usepackage{amsmath,amssymb,amsthm}
\usepackage{mathtools}
\usepackage{graphicx}
\usepackage{booktabs}
\usepackage{enumitem}
\usepackage{microtype}
\usepackage{lingstyle}

\newtheorem{proposition}{Proposition}
\newtheorem{definition}{Definition}
\title{Frictive Policy Optimization for LLMs:\\
Epistemic Intervention, Risk-Sensitive Control, and Reflective Alignment}

\author{
James Pustejovsky \\
Brandeis University \\
Waltham, MA, USA \\
\texttt{jamesp@brandeis.edu}
\And
Nikhil Krishnaswamy \\
Colorado State University \\
Fort Collins, CO, USA \\
\texttt{nkrishna@colostate.edu}
}
\date{}

\begin{document}
\maketitle
\begin{abstract}
We propose {\it Frictive Policy Optimization} (FPO), a framework for learning language model policies that regulate not only what to say, but {\it when and how to intervene} in order to manage epistemic and normative risk. 
Unlike standard alignment methods that optimize surface-level preference or task utility, FPO treats clarification, verification, challenge, redirection, and refusal as explicit control actions whose purpose is to shape the evolution of belief, commitment, and uncertainty over time. 
We formalize alignment as a risk-sensitive epistemic control problem in which intervention decisions are selected based on their expected effect on downstream epistemic quality rather than on immediate reward alone.

We introduce a compact taxonomy of frictive interventions, a structured {\it friction functional} that operationalizes multiple alignment failure modes, and a unified family of FPO methods spanning reward shaping, preference pairing, group-relative ranking, and risk-conditioned trust regions. 
We further propose an evaluation framework that measures epistemic competence directly through clarification behavior, calibration, contradiction repair, refusal proportionality, and information efficiency. 
Together, these results provide a formal and algorithmic foundation for learning agents that are aligned not only in outcome, but in epistemic conduct.
\end{abstract}

\section{Introduction}
\label{sec:intro}

Large language models (LLMs) are increasingly deployed not merely as generators of text, but as interactive agents expected to collaborate with users over extended dialogues. In such settings, success depends not only on producing fluent or informative responses, but on exercising epistemic judgment: recognizing when a request is underspecified, when assumptions require verification, when a response may be unsafe or normatively problematic, and when contradictions or errors must be repaired over the course of an interaction.

Despite dramatic improvements in instruction following and preference alignment, contemporary alignment paradigms remain poorly equipped to support this form of epistemic competence. Methods such as reinforcement learning from human feedback (RLHF), direct preference optimization (DPO), and group-relative policy optimization (GRPO) primarily optimize for compliance with observed preferences or rankings. They implicitly assume that a response is always required and that the principal challenge is selecting the \textit{best} response among alternatives. As a consequence, these methods systematically conflate epistemic responsibility with surface-level helpfulness.

This limitation manifests in well-documented failure modes. Aligned models frequently respond confidently to underspecified or ambiguous requests, produce answers grounded in unverified assumptions, comply with ill-posed or hazardous instructions, and struggle to repair contradictions across turns. These behaviors are not anomalous; they are structurally induced by training objectives that lack any representation of \textit{intervention} as a legitimate action. From the perspective of such objectives, hesitation, clarification, and refusal appear only as suboptimal answers rather than as rational control decisions.

Recent empirical work reinforces this diagnosis by showing that LLMs, despite their strong generative abilities, systematically struggle with interactional timing -- for example, failing to predict appropriate moments to speak in unscripted dialogue \citep{umair2024large}. This highlights a fundamental limitation of alignment methods that optimize only what to say, rather than when and whether to speak at all.

In this paper, we argue that alignment failures of this kind reflect a deeper conceptual gap: the absence of a principled account of \textit{epistemic intervention}. Human collaborators routinely resist immediate task completion in order to improve the epistemic quality of an interaction. They ask clarifying questions, challenge premises, request evidence, redirect ill-posed tasks, or refuse requests that violate safety or normative constraints. Crucially, these behaviors are not expressions of uncertainty alone; they are deliberate acts that regulate commitment, manage risk, and scaffold joint understanding.

We adopt the notion of \textit{friction} not as an obstacle to optimal behavior, but as a constructive control signal that regulates when and how an agent should act under epistemic and normative uncertainty. In human dialogue, productive interaction often depends on well-timed hesitation, clarification, or refusal, rather than immediate task completion. This perspective reframes alignment not as maximizing instantaneous helpfulness, but as managing the dynamics of belief, risk, and coordination over time.

To address this gap, we introduce \textit{Frictive Policy Optimization} (FPO), a general framework for learning policies that strategically deploy epistemic friction. FPO treats intervention types -- such as clarification, verification, challenge, redirection, and refusal -- as first-class actions. These actions are selected based on their expected impact on epistemic and normative risk, rather than on surface-level preference alone. At the core of FPO is a structured \textit{friction functional} that quantifies epistemic failure modes, including uncertainty miscalibration, contradiction with established context, safety hazards, value conflicts, and expected information gain. This functional serves as a unifying signal across multiple optimization paradigms, enabling friction to be incorporated into reward-based, preference-based, and ranking-based learning methods.

Importantly, FPO is not a single algorithm. It is a family of methods unified by a common theoretical commitment: that resistance to immediate task completion can be rational, learnable, and essential for alignment. This paper builds on a line of work initiated in \citet{pustejovsky2025fpo}, where Frictive Policy Optimization was first introduced as a framework for incorporating epistemic friction into the alignment of large language model agents. That initial work proposed three concrete algorithmic instantiations—Friction-Augmented Rewards (FAR), Friction Preference Pairing (FPP), and Group-Relative Frictive Ranking (GRFR)—and demonstrated their feasibility in agent interaction settings.

The present paper substantially extends this framework along three dimensions. First, we provide a formal treatment of epistemic risk and its surrogate approximation by a structured friction functional. Second, we generalize the original algorithms into a unified control-theoretic framework for learning intervention policies. Third, we introduce a new method, \textit{Friction-Conditioned Trust Regions} (FTR), which regulates policy plasticity itself as a function of epistemic risk. Together, these extensions transform FPO from a heuristic proposal into a principled family of risk-sensitive intervention learning methods.

This paper makes five main contributions. (i) We formalize epistemic intervention in dialogue as a risk-sensitive control problem, in which agents must choose not only what to say, but when to clarify, refuse, defer, or repair. (ii) We define a structured friction functional that serves as a tractable surrogate for latent epistemic risk, decomposing multiple epistemic and normative failure modes into measurable components. (iii) We introduce a unified family of Frictive Policy Optimization methods—FAR, FPP, GRFR, and FTR—that incorporate epistemic friction into learning through rewards, preferences, rankings, and risk-conditioned policy constraints. (iv) We provide formal objectives, optimization procedures, and theoretical analysis showing when and why epistemic interventions such as clarification are optimal under the proposed framework. (v) We propose a trajectory-based evaluation framework that measures clarification competence, repair behavior, proportional refusal, and information efficiency, rather than surface-level helpfulness alone.

Conceptually, this work reframes alignment as a problem of epistemic control rather than static preference satisfaction. Algorithmically, it provides concrete learning methods for inducing selective, risk-calibrated intervention behavior in language models. Our goal is to establish a principled foundation for training and evaluating models that regulate commitment, manage uncertainty, and intervene responsibly in long-horizon interaction.

\section{Related Work}
\label{sec:related-work}

This section situates Frictive Policy Optimization (FPO) within four converging lines of research: (i) abstention, deferral, and selective answering in large language models; (ii) epistemic alignment and self-regulation in interactive systems; (iii) decision-theoretic models of clarification and dialogue control; and (iv) risk-sensitive and constrained reinforcement learning. Across these literatures, a common concern is how intelligent agents should regulate commitment under uncertainty—when to answer, when to defer, and when to resist a request altogether. Our framework unifies these themes by treating epistemic intervention as a first-class control problem and by providing a learning-based account of how such interventions can be optimized.

Frictive Policy Optimization (FPO) was first introduced in \citet{pustejovsky2025fpo}, where it was argued that alignment methods based solely on task reward or preference learning neglect a critical epistemic dimension of human--AI collaboration. That paper introduced the core metaphor of \textit{friction} and proposed three initial algorithmic variants: Friction-Augmented Rewards (FAR), Friction Preference Pairing (FPP), and Group-Relative Frictive Ranking (GRFR). Subsequent follow-up work has explored extensions of this idea to epistemic alignment and structured intervention policies \citep{obiso2025dynamic,nath2025learning,nath2025frictional,nath2026collaborate}. The present paper unifies and generalizes these strands by (i) formalizing latent epistemic risk and its surrogate approximation, (ii) providing a principled decomposition of friction into productive and unproductive components, and (iii) introducing a new class of risk-conditioned regularization methods (FTR) that operate at the level of policy updates rather than learning signals. In contrast to earlier work, which treated friction primarily as an auxiliary reward or preference signal, the framework developed here treats intervention choice as a first-class control problem and situates FPO within the broader literature on risk-sensitive and constrained policy optimization.

Early work on refusal and deferral in dialogue focused on grounding and repair mechanisms in collaborative conversation \citep{traum1994computational,clark1996using}. More recent LLM-oriented work treats abstention as a reliability and safety mechanism for selective prediction, rejection, and refusal. A comprehensive survey by \citet{wen2025know} categorizes abstention methods by lifecycle stage (pretraining, alignment, inference) and proposes standardized evaluation frameworks. We differ from this literature in two respects. First, we treat abstention and refusal not as special-case safety filters but as first-class control actions within a unified policy. Second, we optimize abstention jointly with clarification and self-correction under a single risk-sensitive objective, rather than introducing abstention only as an inference-time heuristic. Recent work on input clarification and uncertainty decomposition \citep{hou2024input} further supports the view that clarification is an active mechanism for managing epistemic uncertainty rather than a post-hoc repair.

A growing body of work frames alignment in explicitly epistemic terms. \citet{clark2025epistemic} propose an Epistemic Alignment Framework that identifies challenges such as ambiguity resolution, abstention, uncertainty expression, and self-correction as central to aligned behavior. Constitutional AI \citep{bai2022constitutional} operationalizes self-regulation through post-generation critique and revision guided by normative principles. Our framework is complementary but differs in formalization. Whereas Constitutional AI treats critique and revision as post-hoc correction mechanisms, we treat clarification, refusal, and self-correction as actions chosen directly by a control policy. Likewise, while epistemic alignment frameworks provide conceptual taxonomies of challenges, we provide a unified decision-theoretic objective and concrete optimization methods for learning epistemic interventions.

Clarification has long been studied in probabilistic dialogue systems and decision-theoretic models. In POMDP-based dialogue management, belief states summarize interaction histories and policies are optimized over these latent states \citep{young2013pomdp,williams2007partially,traum1994computational}. Decision-theoretic models of question generation explicitly optimize expected information gain or entropy reduction to decide when to ask clarifying questions \citep{gervits2021decision}. Recent work such as SAGE-Agent formulates clarification as a POMDP with expected value of information objectives and trains policies that reward clarification under uncertainty \citep{suri2025structured}. Our formulation generalizes this line of work by embedding clarification within a broader friction functional that also captures contradiction, hazard, and value conflict, and by optimizing clarification jointly with refusal and self-correction within a single policy.

Our objective also builds on a large literature on risk-sensitive and robust reinforcement learning. Classical work on variance-aware and CVaR-based objectives develops criteria for optimizing tail risk and worst-case performance \citep{tamar2015policy,chow2015risk}. More recent work provides general reductions from optimized certainty equivalents to standard RL algorithms \citep{Wang2024ReductionsRSRL}. We differ from this literature in the nature of the risk being optimized. Whereas risk-sensitive RL typically defines risk over return distributions, we define epistemic and normative risk over belief states and dialogue trajectories. This allows risk sensitivity to govern intervention behavior such as clarification, refusal, and self-correction, rather than only the variability of task returns.

Across these literatures, prior work has addressed abstention, clarification, self-correction, and risk sensitivity largely in isolation. Frictive Policy Optimization unifies these threads by treating epistemic interventions as control actions optimized under a single risk-sensitive objective, providing both a formal foundation and a family of practical learning algorithms for epistemically aligned behavior.

\section{Frictive Interventions in Dialogue}
\label{sec:frictive-interventions}

This section introduces a taxonomy of {\it frictive interventions} in dialogue.
The taxonomy serves two roles.
First, it provides a principled conceptual account of the kinds of epistemic resistance that arise naturally in human collaborative interaction.
Second, it defines the {\it action space} over which Frictive Policy Optimization operates.
Unlike surface-level stylistic variation, frictive interventions are qualitatively distinct communicative acts that regulate commitment, manage uncertainty, and control epistemic risk.

\subsection{From Dialogue Acts to Epistemic Interventions}

Dialogue research has long recognized distinctions among communicative acts, such as questions, assertions, requests, and acknowledgments.
While these distinctions are essential for modeling conversational structure, they are insufficient for alignment in large language models.
Dialogue acts primarily categorize utterances by their conversational role; they do not capture the {\it epistemic motivation} behind resisting or deferring task completion.

Frictive interventions cut across traditional dialogue act categories.
For example, a clarifying question and a rhetorical challenge may both be interrogatives, yet they differ fundamentally in their epistemic function.
Similarly, a refusal may take the syntactic form of an assertion, but its role is not to convey information, but to block an action.

We therefore distinguish frictive interventions from dialogue acts in the following sense:

\begin{quote}
{\it A frictive intervention is defined not by its surface form, but by its role in regulating epistemic commitment and risk within an interaction.}
\end{quote}

This formulation resonates with formal models of shared knowledge and commitment in dialogue \citep{khebour2024common,obiso2025dynamic}, in which interaction state encodes not only surface history but agents’ evolving epistemic commitments and mutual beliefs.

This distinction is also crucial for alignment: a model that merely learns dialogue act frequencies cannot reliably learn when to clarify, challenge, or refuse.
What is required is a taxonomy grounded in epistemic function. Furthermore, this  observation is supported by recent findings that current language models often fail at precisely this kind of interactional timing, even when they generate locally plausible responses; for instance, models struggle to predict appropriate opportunities to speak in naturalistic dialogue settings \citep{umair2024large}. Such results underscore the need to model intervention decisions—whether to speak, defer, clarify, or refuse—as explicit control actions rather than as byproducts of surface-level generation.

\subsection{Definition and Core Classes of Frictive Intervention}
\label{sec:friction-types}
We define a frictive intervention as follows.

\begin{definition}[Frictive Intervention]
A {\it frictive intervention} is an intentional communicative act that temporarily resists immediate task completion in order to reduce epistemic, normative, or safety risk in an interaction.
\end{definition}

We distinguish a small set of core frictive intervention types that serve as the action space for Frictive Policy Optimization. These classes are not intended as an exhaustive theory of dialogue, but as a minimal set of epistemic control actions sufficient to capture the dominant alignment failure modes in current language models:

\begin{itemize}[leftmargin=*]
\item \textbf{Clarification}, which seek missing constraints or resolve underspecification before acting.
\item \textbf{Verification}, which check the validity of assumptions, premises, or sources.
\item \textbf{Challenge}, which question or reject false or inconsistent presuppositions.
\item \textbf{Redirection}, which reframe ill-posed tasks toward coherent alternatives.
\item \textbf{Refusal}, which block actions that violate safety, ethical, or normative constraints.
\item \textbf{Meta-dialogical interventions}, which explicitly manage the model’s own uncertainty or limitations.
\end{itemize}

Each class is defined functionally by its role in regulating epistemic commitment and risk, rather than by its surface linguistic form. In all cases, the defining property is that the intervention temporarily resists immediate task completion in order to improve downstream epistemic or normative outcomes.

This definition has three important components:

\begin{itemize}[leftmargin=*]
\item \textbf{Intentionality}: The intervention is not noise or uncertainty leakage, but a deliberate control decision.
\item \textbf{Resistance}: The intervention delays, redirects, or blocks immediate task execution.
\item \textbf{Risk reduction}: The purpose of the intervention is to improve epistemic quality or prevent downstream failure.
\end{itemize}

Under this definition, hesitation without purpose, verbosity, or generic hedging do {\it not} count as frictive interventions.
Friction is not a matter of style, but of {\bf control}.

A full formal taxonomy and fine-grained subtypes is presented in \citet{pust-zhu2026}, to which we refer the reader for detailed classification and annotation guidelines. In the present paper, we use these classes only to define the action space over which intervention policies are learned.

\subsection{Frictive Interventions as Actions}

A central claim of this paper is that frictive interventions must be treated as {\it actions} rather than stylistic variants of answers.
This claim has three implications:

\begin{enumerate}[leftmargin=*]
\item Interventions change the {\it state of the interaction}, not just its wording.
\item Interventions incur explicit costs (time, user effort, refusal penalties).
\item Interventions have delayed effects on epistemic risk.
\end{enumerate}

These properties cannot be captured by token-level preferences or response re-ranking alone.
They require a control-theoretic treatment in which the model decides {\it whether} to act before deciding {\it how} to phrase the act.

The taxonomy introduced here defines the action space $\mathcal{A}$ used in subsequent sections.
Each frictive intervention corresponds to a distinct action type that can be selected by a policy based on its expected epistemic consequences.

This move—from descriptive taxonomy to formal action space—marks the transition from conceptual framing to formal modeling.
In the next section, we show how these intervention types are embedded in a risk-sensitive epistemic control framework.
\section{Epistemic Friction as Risk-Sensitive Control}
\label{sec:epistemic-control}
We model dialogue as a controlled stochastic process in which the agent regulates epistemic commitment under uncertainty. 
Let $h_t \in \mathcal{H}$ denote the observable interaction history up to time $t$, including the dialogue transcript, multimodal observations, task context, and any available external signals.
We treat $h_t$ as the sufficient interaction state on which policies, risk estimates, and friction functionals are conditioned.
Let $b_t$ denote the agent’s belief state at time $t$, summarizing its uncertainty over latent world states.


\paragraph{Belief Update.}
We assume beliefs evolve according to a Bayesian or learned update rule:
\enumsentence{
$b_{t+1} = \mathrm{BeliefUpdate}(b_t, h_{t+1})$
}

\noindent This belief update equation places our framework within the classical theory of control under partial observability, in which the belief state $b_t$ is treated as a sufficient statistic for the interaction history. In partially observable Markov decision processes (POMDPs), optimal control is performed over evolving belief distributions rather than over raw observations \citep{astrom1965optimal,smallwood1973optimal,kaelbling1998planning}. 

In dialogue systems, this formulation underlies probabilistic dialogue state tracking and grounding, where the agent maintains and updates an internal epistemic state in response to user observations \citep{traum1994computational,clark1996using,williams2007partially}. More recent neural dialogue managers similarly operate by learning approximate belief update rules from interaction data \citep{henderson2014word,lee2019sumbt}. Our formulation makes this assumption explicit: alignment is a control problem over latent epistemic states, not merely a mapping from histories to surface responses.

This equation defines the epistemic state transition: after observing the new dialogue history $h_{t+1}$, the agent updates its belief state from $b_t$ to $b_{t+1}$. This makes explicit that alignment is a control problem over evolving beliefs, not just over surface responses.

\paragraph{Risk-Sensitive Objective.}
We define the agent’s objective as a discounted risk-sensitive return. Let $y_t$ denote the surface realization (utterance) generated at time $t$. We assume that $y_t$ is conditionally sampled from a realization policy $\pi_{\mathrm{gen}}(y_t \mid h_t, a_t)$ after selecting an intervention type $a_t$, as formalized below in Eq.~(\ref{eq:objective}).

\enumsentence{
$
J(\pi)
= $ \\
$\mathbb{E}_\pi \Bigg[
\sum_{t=0}^{T}\gamma^t \Big(
R_{\text{task}}(h_t,a_t,y_t)  \; -
 \\
\lambda\,C_{\text{fric}}(a_t) - \eta\,\mathcal{R}_{\text{risk}}(b_t,h_t,a_t,y_t)
\Big)
\Bigg]$
}\label{eq:objective}

This objective follows the standard discounted reinforcement learning formulation \citep{puterman1994markov,kaelbling1998planning}, augmented with an explicit risk-sensitive penalty as in classical risk-sensitive and constrained Markov decision processes \citep{howard1972risk,tamar2015policy,chow2015risk,achiam2017constrained}. In the service of FPO, this objective defines the central tradeoff of the model: the policy $\pi$ maximizes discounted task reward while subtracting both an immediate intervention cost $C_{\text{fric}}$ and an epistemic risk functional $\mathcal{R}_{\text{risk}}$. The weights $\lambda$ and $\eta$ control how strongly the agent prefers low-friction and low-risk trajectories relative to short-term task success.

The terms have the following interpretation:

\begin{itemize}
\item $R_{\text{task}}$ measures immediate task utility from producing content.
\item $C_{\text{fric}}(a_t)$ is the intervention cost, encoding user burden, latency, or refusal penalties.
\item $\mathcal{R}_{\text{risk}}(b_t,h_t,a_t,y_t)$ measures expected downstream epistemic or normative failure under belief state $b_t$.
\end{itemize}

In what follows, when the belief state $b_t$ is deterministically inferred from the observable history $h_t$, we suppress the explicit dependence on $b_t$ and write $\mathcal{R}_{\text{risk}}(h_t, a_t, y_t)$ for notational simplicity.

\paragraph{Risk as Belief-Weighted Expectation.}
We model epistemic risk as an expectation over latent states:

\enumsentence{
$\mathcal{R}_{\text{risk}}(b_t,h_t,a_t,y_t)
= $ \\
\hspace*{29mm} $\mathbb{E}_{s \sim b_t}
\big[
r_{\text{risk}}(s, h_t, a_t, y_t)
\big]$}

This is the standard belief-state cost formulation in POMDPs \citep{astrom1965optimal,kaelbling1998planning,young2013pomdp}, here interpreted as epistemic and normative risk. Here we define epistemic risk as the belief-weighted expectation of a latent risk function $r_{\text{risk}}$. It makes explicit that risk depends on unobserved world states $s$, and that the agent must reason under epistemic uncertainty rather than optimizing a fully observed cost.

\paragraph{Policy Factorization.}
We factor the policy into intervention type and surface realization:
\enumsentence{
$\pi(a_t, y_t \mid h_t)
= $ \\
\hspace*{21mm} $\pi_{\text{int}}(a_t \mid h_t)\,
\pi_{\text{gen}}(y_t \mid h_t, a_t)$
}

This factorization mirrors hierarchical dialogue policies and option-based control, in which a high-level decision selects an intervention type and a conditional generator realizes it \citep{williams2007partially,young2013pomdp,sutton1999between,bacon2017option}. In this formulation, it separates the decision of {\it what kind of epistemic intervention} to perform (answer, clarify, refuse, etc.) from the decision of {\it how to realize it linguistically}. This is crucial for modeling intervention as a first-class control action.

\paragraph{Clarification as Optimal Epistemic Control}

The FPO framework treats intervention type as a first-class decision variable, but this raises a natural normative question: under what conditions is clarification actually the optimal action? If intervention costs always dominate, the framework collapses to standard response generation; if clarification is always preferred, it degenerates into pathological hesitation. We therefore establish a simple dominance result showing that clarification is optimal precisely when its expected epistemic benefit outweighs its interaction cost. This yields an explicit threshold criterion that the FPO methods are designed to learn to approximate from data.

\begin{proposition}[Threshold Optimality of Clarification]
\label{prop:clarification-optimal}
Consider a decision point at time $t$ with belief state $b_t$ and history $h_t$. Let $a_c$ denote a clarificatory action and $a_a$ denote a direct answer action. Suppose:
\begin{enumerate}[label=(\roman*),nosep]
\item \textbf{Risk reduction:} Clarification yields strictly lower expected future epistemic risk:
$$
\mathbb{E}[\mathcal{R}_{\mathrm{risk}}(b_{t+1}) \mid a_c]
<
\mathbb{E}[\mathcal{R}_{\mathrm{risk}}(b_{t+1}) \mid a_a].
$$

\item \textbf{Bounded friction cost:} The clarification cost is finite:
$
C_{\mathrm{fric}}(a_c) < \infty.
$
\end{enumerate}
Then clarification strictly dominates direct answering whenever:
$$\lambda \, C_{\mathrm{fric}}(a_c) < $$
$$\eta \left( \mathbb{E}[\mathcal{R}_{\mathrm{risk}}(b_{t+1}) \mid a_a] - \mathbb{E}[\mathcal{R}_{\mathrm{risk}}(b_{t+1}) \mid a_c] \right)$$

where $\lambda$ and $\eta$ are the objective weights on interaction cost and epistemic risk, respectively.
\end{proposition}

\begin{proof}[Proof sketch]
Under the risk-sensitive objective~(\ref{eq:objective}), the Q-value of an action decomposes into immediate task reward, weighted epistemic risk, and weighted intervention cost. By condition (i), clarification reduces the expected future risk by some $\Delta_{\mathrm{risk}} > 0$ relative to direct answering. This yields a gain of $\eta \Delta_{\mathrm{risk}}$ in the objective. Clarification also incurs an immediate penalty of $\lambda C_{\mathrm{fric}}(a_c)$. When the former exceeds the latter, the clarificatory action has strictly higher expected return. The dominance condition follows by rearranging terms, and optimality follows from Bellman’s principle of optimality.
\end{proof}

\paragraph{Interpretation.}
Proposition~\ref{prop:clarification-optimal} formalizes the central tradeoff that Frictive Policy Optimization is designed to learn. The ratio $\eta/\lambda$ defines an implicit clarification threshold: it specifies how much epistemic risk reduction is required to justify the interaction cost of delaying task completion. When $\eta \gg \lambda$, even modest expected risk reduction warrants clarification; when $\lambda \gg \eta$, only severe uncertainty or hazard justifies intervention.

Crucially, this result identifies what the learning problem actually is. The agent does not need direct access to the latent belief state $b_t$ or the true risk functional $\mathcal{R}_{\mathrm{risk}}$. Instead, the FPO methods (FAR, FPP, GRFR, and FTR) are designed to learn a soft, context-sensitive approximation to this threshold using the tractable friction surrogate $F(h,y)$. In this sense, Proposition~\ref{prop:clarification-optimal} provides the normative target that FPO approximates: it characterizes the regime in which clarification is not merely permissible or stylistically cautious, but strictly optimal under risk-sensitive control.

 \section{The Friction Functional}
\label{sec:friction-functional}

We now introduce a tractable surrogate for epistemic risk. Since the latent risk functional 
$\mathcal{R}_{\text{risk}}(b_t,h_t,a_t,y_t)$ is not directly observable, we approximate it by a structured friction functional defined over dialogue histories and candidate responses.

Before introducing the specific components of the friction functional, it is important to clarify the design principles that govern its construction. The friction functional is not intended as an arbitrary collection of heuristics, nor as an exhaustive catalogue of all possible dialogue failures. Rather, it is a structured surrogate designed to capture a small number of {\it irreducible epistemic failure modes} that recur across aligned interaction.

We select components according to three criteria. First, each component must correspond to a {\it distinct epistemic pathology} that is known to induce downstream failure in decision making: miscalibration, inconsistency, latent hazard, silent intent fixing, and unresolved uncertainty \citep{munn2024truth,devilling2025polite}. Second, each component must be {\it operationalizable} from observable signals available at training or inference time, without access to the latent belief state or true risk functional. Third, each component must admit a {\it monotonic interpretation}: increasing values should reliably indicate increasing epistemic danger or normative violation.

Under these criteria, the friction functional is designed to approximate not the full space of dialogue errors, but the minimal set of failure modes that systematically corrupt belief quality over time. {\bf Miscalibration} captures failures of confidence management; {\bf contradiction} captures failures of logical coherence; {\bf hazard} captures failures of normative and safety compliance; {\bf value conflict} captures failures of intent alignment; and {\bf information gain} captures the unique class of interventions that deliberately increase short-term interaction cost in order to reduce long-term epistemic risk.

This design yields a decomposition in which each term corresponds to a theoretically motivated axis of epistemic control, rather than an ad hoc feature. The resulting functional is modular, interpretable, and extensible: additional components may be introduced when new epistemic failure modes are identified, but the present set reflects a minimal basis sufficient to drive risk-sensitive intervention in practice.

These components are chosen to mirror the core intervention classes introduced in Section~\ref{sec:friction-types}, providing a direct mapping from qualitative taxonomy to quantitative control signals.

\subsection{Decomposition of Friction}

We now make explicit how the friction functional is constructed from a small set of core epistemic failure modes. As argued in Section~\ref{sec:frictive-interventions}, frictive interventions in dialogue are triggered not by arbitrary surface errors, but by a limited number of recurrent pathologies that systematically degrade belief quality and downstream decision making. These include miscalibration of confidence, logical inconsistency, latent hazard, silent intent fixing, and unresolved uncertainty.

Following the taxonomy of Section~\ref{sec:frictive-interventions}, we distinguish between \textbf{unproductive friction} -- epistemic and normative failures that increase downstream risk -- and \textbf{productive friction} -- deliberate interventions that reduce long-term uncertainty at the cost of short-term interaction burden. This distinction is fundamental: unproductive friction should be minimized, while productive friction should be encouraged when epistemic risk is high.

We first define the \textbf{unproductive friction} components, each corresponding to a distinct epistemic failure mode:
\enumsentence{\label{eq:unproductive-friction}
$F^{-}(h,y)
= $ \\
\hspace*{15mm}$w_1\,\mathrm{Unc}(h,y)
+
w_2\,\mathrm{Contr}(h,y) \; +$ \\
\hspace*{15mm}$ 
w_3\,\mathrm{Haz}(h,y)
+
w_4\,\mathrm{ValConf}(h,y)$
}
This aggregates miscalibration, contradiction, hazard, and value conflict, the four core pathologies that systematically degrade belief quality over time \citep{gabriel2020artificial,linteaching,bai2022constitutional,ganguli2022red}. High values of $F^{-}$ indicate epistemically dangerous responses that should be avoided.

We then define \textbf{productive friction} as the information gain component:
\enumsentence{\label{eq:productive-friction}
$F^{+}(h,y)
=
w_5\,\mathrm{InfoGain}(h,y)$
}
This captures interventions that deliberately increase short-term interaction cost in order to reduce long-term epistemic uncertainty. High values of $F^{+}$ indicate epistemically beneficial interventions such as clarification and verification.

Finally, we define the \textbf{net friction functional} as the difference:
\enumsentence{\label{eq:net-friction}
$F(h,y) = F^{+}(h,y) - F^{-}(h,y)$
}
This measures the overall epistemic contribution of a response: positive values indicate net epistemic benefit (productive friction dominates), while negative values indicate net epistemic harm (unproductive friction dominates). The net friction functional thus provides a single scalar summary of epistemic quality that can be used directly in reward shaping, preference learning, and trajectory ranking.

This benefit--loss decomposition mirrors standard formulations of net utility in risk-sensitive and cost-sensitive control \citep{bertsekas1995dynamic,sutton2018reinforcement}, and follows the additive decomposition of shaped rewards and auxiliary costs in reinforcement learning \citep{ng1999policy}. The key insight is that the five component scores define a low-dimensional epistemic control space: any intervention that degrades belief quality must manifest through at least one axis of $F^{-}$, and any beneficial intervention must either reduce $F^{-}$ or increase $F^{+}$.

\paragraph{Uncertainty and Miscalibration.}

Following standard miscalibration measures used in expected calibration error for probabilistic models \citep{niculescu2005predicting,guo2017calibration}, we define the uncertainty component as:
\enumsentence{\label{eq:miscalibration}
$\mathrm{Unc}(h,y)
=
\big|
\mathrm{conf}(y \mid h)
-
\mathrm{acc}(\hat y, h)
\big|$
}
This measures miscalibration: the absolute difference between the model's expressed confidence in response $y$ and the estimated probability that $y$ is correct given the dialogue history. High values indicate epistemically dangerous overconfidence or underconfidence. As a component of $F^{-}$, miscalibration contributes to unproductive friction and should be minimized.

\paragraph{Contradiction.}

We penalize explicit contradictions with prior commitments by operationalizing contradiction using standard natural language inference models trained for entailment and contradiction detection \citep{bowman2015snli,williams2018mnli}.
\enumsentence{\label{eq:contradiction}
$\mathrm{Contr}(h,y)
=
\max\{0,\, \mathrm{NLI}_{\text{contr}}(y,h) - \tau_{\text{contr}}\}$
}
This assigns friction when the response $y$ contradicts the dialogue history according to a natural language inference model, beyond a tolerance threshold $\tau_{\text{contr}}$. It operationalizes logical inconsistency as an epistemic risk signal. As a component of $F^{-}$, contradiction contributes to unproductive friction.

\paragraph{Hazard and Normative Violation.}

We define a hazard component for safety and policy violations, following standard scoring techniques for safety classification and severity-weighted risk as used in aligned language models \citep{weidinger2021ethical,bai2022constitutional}.
\enumsentence{\label{eq:hazard}
$\mathrm{Haz}(h,y)
= 
\mathbf{1}_{\text{haz}}(h,y)\cdot \mathrm{severity}(h,y)$
}
This assigns friction proportional to the detected severity of safety, legal, or policy violations triggered by $y$, making high-risk normative failures dominate the friction score. As a component of $F^{-}$, hazard contributes to unproductive friction.

\paragraph{Value Conflict and Silent Intent Fixing.}

We measure value conflict as a shift in inferred user intent. This is similar to Bayesian belief update and intent tracking in dialogue, where KL divergence is used as a penalty for forced belief revision \citep{williams2007partially,traum2003information}.
\enumsentence{\label{eq:value-conflict}
$\mathrm{ValConf}(h,y)
=
\mathrm{KL}\big(
P(\iota \mid h)
\,\|\, 
P(\iota \mid h,y)
\big)$
}
Specifically, this measures how much committing to response $y$ forces a revision of the inferred user intent distribution. Large KL divergence indicates silent intent fixing, where the model imposes an interpretation rather than resolving ambiguity. As a component of $F^{-}$, value conflict contributes to unproductive friction.

\paragraph{Information Gain.}

We reward epistemically informative interventions as follows \citep{lindley1956measure,gervits2021decision,hou2024input}:
\enumsentence{\label{eq:info-gain}
$\mathrm{InfoGain}(h,a)
=
\mathbb{E}_{o \sim P(\cdot \mid h,a)}
\big[
H(b_t) - H(b_{t+1})
\big]$
}
This measures the expected reduction in belief entropy induced by taking intervention $a$, rewarding actions that actively reduce epistemic uncertainty. Unlike the preceding components, information gain is {\it productive}: it captures deliberate epistemic interventions (clarification, verification, information seeking) that trade short-term interaction cost for long-term risk reduction. Accordingly, $\mathrm{InfoGain}$ is the sole component of $F^{+}$ and enters the net friction functional with a positive sign.

\subsection{From Friction to Risk}
\label{sec:friction-to-risk}

We now relate friction to latent epistemic risk. The central modeling assumption of this framework is that {\it unproductive friction} serves as a tractable surrogate for the unobservable latent epistemic risk:
\enumsentence{\label{eq:surrogate-risk}
$F^{-}(h,y)
\;\approx\;
\mathcal{R}_{\text{risk}}(b_t,h_t,a_t,y_t)$
}
This follows the standard use of surrogate reward and cost functionals in inverse and preference-based reinforcement learning \citep{ng2000algorithms,christiano2017deep,rafailov2023direct}.

Equivalently, the net friction functional approximates the {\it negation} of epistemic risk, offset by productive friction:
\enumsentence{\label{eq:surrogate-net}
$F(h,y)
\;\approx\;
F^{+}(h,y) - \mathcal{R}_{\text{risk}}(b_t,h_t,a_t,y_t)$
}
This relationship has a natural interpretation: responses with high net friction $F$ are those that maximize information gain while minimizing epistemic failure modes---precisely the responses that reduce downstream risk. Consequently, {\it higher values of $F$ indicate better epistemic quality}, and adding $F$ to a reward signal (as in Friction-Augmented Rewards) has the correct effect of encouraging epistemically beneficial behavior.

Importantly, this approximation is not intended to assert that friction is identical to epistemic risk in any literal or ground-truth sense. Rather, it asserts that a suitably constructed friction functional provides a {\it behaviorally sufficient proxy} for the aspects of epistemic risk that matter for policy learning. The key requirement is not pointwise accuracy of $F^{-}(h,y)$ as an estimator of $\mathcal{R}_{\text{risk}}$, but {\it policy consistency}: actions that reduce latent epistemic risk should, in expectation, also reduce unproductive friction, and vice versa.

More formally, the approximation in (\ref{eq:surrogate-risk}) is justified under the standard assumption of {\it surrogate consistency}: there exists a monotone transformation $\phi$ such that, for all contexts $h$ and admissible actions $(a,y)$,

\enumsentence{\label{eq:surrogate-consistency}
$\mathcal{R}_{\text{risk}}(b_t,h_t,a_t,y_t)
< \mathcal{R}_{\text{risk}}(b_t,h_t,a'_t,y'_t)$ \\
$\Rightarrow \quad
F^{-}(h,y) < F^{-}(h,y')$
}

That is, the unproductive friction functional need only preserve the ordering of interventions by latent epistemic risk, not their absolute scale. Under this assumption, minimizing expected unproductive friction induces the same set of optimal policies as minimizing the true latent risk, up to equivalence classes of policies that are indistinguishable under the ordering \citep{ng1999policy,ng2000algorithms,rafailov2023direct}.

Furthermore, the composite nature of $F(h,y)$ provides a natural defense  surrogate over-optimization (Goodhart's Law). Because the functional aggregates structurally distinct failure modes -- logical contradiction, intent drift, miscalibration, and hazard -- adversarial exploitation of any single 
component is likely to degrade performance on others. For instance, a policy that games the NLI classifier to avoid contradiction penalties by producing vacuous or generic text will incur penalties from the InfoGain or Value Conflict components. This structural diversity makes the surrogate objective more resistant to single-metric gaming than monolithic reward signals.

We do not claim pointwise bounds on the approximation error $|F^{-}(h,y) - \mathcal{R}_{\text{risk}}|$; rather, we require only that misorderings be sufficiently rare that the induced policy gradient remains aligned with the true risk gradient in expectation.

This places the framework squarely within the tradition of surrogate optimization in reinforcement learning, where unobservable objectives are optimized through structured, learnable proxies that preserve the ordering of actions by long-term utility. Under this view, the success of Frictive Policy Optimization does not depend on recovering the true epistemic risk functional, but on constructing a friction surrogate whose gradients induce interventions that systematically improve belief quality, reduce downstream failure, and stabilize long-horizon epistemic behavior.

In this sense, the friction functional plays the same conceptual role as learned reward models in preference-based reinforcement learning: it mediates between unobservable normative criteria and tractable learning signals, and its adequacy is ultimately validated not by calibration to ground truth, but by the quality of the policies it induces.

\subsection{Interpretation.}

The net friction functional $F(h,y) = F^{+}(h,y) - F^{-}(h,y)$ thus provides a structured, interpretable, and differentiable measure of epistemic quality. Each component corresponds to a distinct failure mode or benefit that can be measured, optimized, and audited independently.

To summarize the decomposition:
\begin{itemize}
\item \textbf{Productive friction} $F^{+}(h,y)$ should be {\it encouraged}: it captures deliberate epistemic interventions (clarification, verification, information seeking) that trade short-term interaction cost for long-term risk reduction.

\item \textbf{Unproductive friction} $F^{-}(h,y)$ should be {\it suppressed}: it captures epistemic and normative failures (miscalibration, contradiction, hazard, silent intent fixing) that increase downstream risk without benefit.
\end{itemize}

This decomposition allows learning algorithms to distinguish between friction that should be encouraged and friction that should be suppressed---a distinction that cannot be captured by a single scalar measure of ``helpfulness'' or ``compliance.'' In the FPO methods that follow, the net friction $F$ enters reward shaping, preference pairing, and trajectory ranking with a positive sign, reflecting the principle that higher net friction corresponds to better epistemic quality.

\section{Frictive Policy Optimization}
\label{sec:fpo-methods}

In this section, we introduce a unified family of {\it Frictive Policy Optimization} (FPO) methods for learning epistemically aware intervention policies in dialogue. The central learning problem is not merely to generate fluent or helpful responses, but to decide {\it when to answer, when to clarify, when to refuse, and when to defer}, as a function of evolving epistemic risk.

Standard alignment methods implicitly assume that all actions are answers, and that epistemic regulation can be imposed externally through decoding heuristics, refusal filters, or post-hoc safety layers. In contrast, FPO treats intervention choice itself as a first-class control problem: the policy must learn not only {\it what} to say, but {\it whether} to speak, {\it how} to intervene, and {\it how strongly} to trade off short-term utility against long-horizon epistemic quality.

All methods in the FPO family instantiate the same theoretical commitment: that epistemic friction should be learned as part of the policy, rather than imposed as a heuristic or post-hoc filter. The methods differ in where friction enters the learning loop—through shaped rewards, preference supervision, trajectory-level ranking, or policy regularization—but they share a common objective: to induce policies whose intervention behavior is calibrated to epistemic risk.

The first three methods (FAR, FPP, and GRFR) differ in how epistemic friction enters the {\it learning signal}: through reward shaping, preference supervision, and group-relative ranking, respectively. Friction-Conditioned Trust Regions (FTR), introduced in Section~\ref{sec:ftr}, operate instead at the level of {\it policy regularization}, constraining how far the policy may deviate from a base model as a function of epistemic risk. As such, FTR is orthogonal to the other methods and can in principle be combined with any of them.

\subsection{General Learning Framework}

We now describe the general learning setting shared by all FPO methods. The goal is to learn a policy that maps dialogue histories not directly to surface responses, but to {\it intervention-structured actions} whose long-horizon epistemic consequences differ qualitatively. Unlike standard language model training, where each training example is a single input--output pair, FPO must learn a policy over a mixed action space consisting of both intervention types and their realizations.

We assume access to the following:
\vspace*{-2mm}

\begin{itemize}[leftmargin=*]
\item A base policy $\pi_0$ (e.g., instruction-tuned or RLHF-trained);
\vspace*{-2mm}
\item A dataset $\mathcal{D} = \{(h, y)\}$ of dialogue contexts and candidate responses;
\vspace*{-2mm}
\item Optional preference annotations or groupings;
\vspace*{-2mm}
\item Auxiliary estimators for components of $F(h,y)$.
\end{itemize}

\noindent As introduced in Section~\ref{sec:epistemic-control}, we factor the policy as:
\enumsentence{
$\pi_\theta(y \mid h)
=
\sum_{a \in \mathcal{A}}
\pi_\theta(a \mid h)\,
\pi_\theta(y \mid h, a)
$
}
where $\mathcal{A}$ is the taxonomy of frictive interventions. We define a finite intervention action space, 
$$
\mathcal{A} = \{\text{answer}, \text{clarify}, \text{verify}, \text{redirect}, \text{refuse}\},
$$
\noindent 
corresponding to the core classes of frictive interventions introduced in Section~\ref{sec:frictive-interventions}.
Here, \textit{answer} is treated as a zero-friction action corresponding to ordinary response generation; including it explicitly in $\mathcal{A}$ allows the intervention policy to represent the choice \emph{not} to intervene as a first-class decision, with $C_{\mathrm{fric}}(\text{answer}) = 0$.

For notational and modeling simplicity, we treat \textit{challenge} as an operational subtype of \textit{verification} and \textit{redirection}, since all three involve questioning or revising problematic premises. Accordingly, we do not introduce \textit{challenge} as a separate atomic action in the formal action space $\mathcal{A}$, although it remains part of the underlying conceptual taxonomy in Section~\ref{sec:frictive-interventions}.

In the interest of space, we do not introduce meta-dialogue as a separate atomic action in the formal action space $\mathcal{A}$. Instead, such behaviors are subsumed under clarification and redirection, since they function operationally as higher-order forms of epistemic repair and intent negotiation.

This factorization makes explicit a distinction that is implicit in most alignment pipelines: the decision of {\it what kind of epistemic intervention to perform} is separated from the decision of {\it how to realize that intervention linguistically}. In standard training, these two decisions are conflated into a single generation step. In FPO, they are disentangled and jointly optimized.
This induces two coupled learning problems:
\begin{enumerate}
\item learning an {\it intervention policy} $\pi_\theta(a \mid h)$,
\vspace*{-2mm}
\item learning a {\it realization policy} $\pi_\theta(y \mid h,a)$.
\end{enumerate}

All FPO methods differ from standard alignment in that they explicitly train the intervention policy, rather than assuming that every action is an answer. This makes intervention choice itself a learned object, rather than a byproduct of decoding heuristics, refusal filters, or post-hoc safety layers.

From a learning-theoretic perspective, there are only a small number of principled ways to incorporate a structured surrogate such as $F(h,y)$ into policy learning. Friction may enter as an additive term in the reward, as a latent signal in preference supervision, as a ranking criterion over trajectories, or as a constraint on policy updates. The four methods introduced below—FAR, FPP, GRFR, and FTR—correspond exactly to these four design points.

Together, they define a minimal and complementary family of algorithms for learning risk-sensitive intervention policies, spanning reward-based, preference-based, ranking-based, and regularization-based alignment.

\subsection{Friction-Augmented Rewards (FAR)}

Friction-Augmented Rewards (FAR), originally introduced in \citet{pustejovsky2025fpo}, extends standard reinforcement learning by adding a gated friction signal to the task reward. In the present formulation, FAR is generalized to incorporate an explicit risk-gating mechanism and a structured friction functional.

\paragraph{Objective.}

Friction-Augmented Rewards extends standard reinforcement learning by adding a gated {\it epistemic friction} signal to the task reward. This reward is designed to selectively encourage epistemic interventions only when they are justified by elevated risk, while explicitly penalizing their interaction cost.

\enumsentence{
$R'(h,a,y)
=  R_{\text{task}}(h,a,y)
+ $ \\
$\alpha\,g(\mathsf{Risk}(h,y))\,F(h,y)
-
\beta\,C_{\text{fric}}(a)$
\label{FAR-objective}
}

\noindent where we assume:
\vspace*{-2mm}
\begin{itemize}[leftmargin=*]
\item $g(\cdot)$ is a monotone gating function, e.g.,\ $g(r) = \sigma(\kappa(r-\tau))$;
\vspace*{-2mm}

\item $\mathsf{Risk}(h,y)$ is a pre-intervention estimate of epistemic risk used to gate friction, aggregating high-level risk indicators for the candidate intervention before it is executed in context $h$;
\vspace*{-2mm}
\item $C_{\text{fric}}(a)$ penalizes intervention cost.
\end{itemize}

We emphasize that the two penalty terms in (\ref{FAR-objective}) play conceptually distinct roles. 
The term $C_{\text{fric}}(a)$ captures {\it pure interaction cost} (latency, user burden, and refusal penalties associated with taking an intervention action) independent of its epistemic outcome. 
In contrast, the friction functional $F(h,y)$ captures the {\it epistemic and normative quality of the resulting response}, including miscalibration, contradiction, hazard, and value conflict. 
Separating these terms prevents productive information-seeking from being conflated with surface interaction cost, and allows the policy to trade off {\it how costly it is to intervene} against {\it how beneficial the intervention is epistemically}.

This objective defines the shaped reward optimized by FAR. 
The agent maximizes the standard task utility $R_{\text{task}}$, but in addition receives a friction-dependent bonus 
$\alpha\, g(\mathsf{Risk}(h,y))\, F(h,y)$ that is activated only in high-risk contexts. 
The explicit penalty $-\beta C_{\text{fric}}(a)$ prevents the policy from overusing epistemic interventions in low-risk situations. 
Together, these terms implement the principle ``use friction when it matters, but pay for it when you do.''

The gating function $g(\mathsf{Risk}(h,y))$ implements {\it risk-conditional shaping}: 
friction is rewarded only when the estimated epistemic risk is high. 
When risk is low, $g \approx 0$ and friction is effectively ignored; when risk is high, $g \approx 1$ and productive friction is actively encouraged. 
This prevents degenerate policies that intervene habitually and enforces selective, calibrated intervention.

\paragraph{Policy Gradient.}

Any policy-gradient or actor--critic method may be used to optimize the shaped objective in (\ref{FAR-objective}) \citep{williams1992simple,sutton2000policy,schulman2017proximal}:

\enumsentence{
$\nabla_\theta J_{\text{FAR}}
= $ \\
\hspace*{15mm}$\mathbb{E}
\left[
\nabla_\theta \log \pi_\theta(a,y \mid h)\,
R'(h,a,y)
\right]
$
}

This is a standard policy-gradient update: it increases probability mass on intervention--utterance pairs $(a,y)$ that yield higher shaped return $R'$. Epistemic behavior is shaped indirectly through the reward, rather than by explicitly optimizing a separate risk-reduction advantage or belief-state objective. In this sense, FAR treats epistemic regulation as a form of reward shaping: risk-sensitive intervention behavior is induced by modifying the scalar learning signal, without altering the structure of the policy class or the optimization algorithm.

\paragraph{Discussion.}

Friction-Augmented Rewards is the simplest member of the FPO family and is directly compatible with existing RLHF and actor--critic pipelines \citep{christiano2017deep,ouyang2022training}. It requires no change to the policy architecture and treats epistemic regulation as a form of reward shaping.

The main advantage of FAR is its generality: any reinforcement learning algorithm that optimizes a scalar reward can incorporate friction with minimal modification. This makes FAR attractive as a baseline method and as a drop-in extension of standard alignment pipelines.

However, FAR inherits the known limitations of reward shaping. First, it is sensitive to reward scaling: inappropriate choices of $\alpha$ and $\beta$ can lead either to excessive intervention or to complete suppression of epistemic behavior. Second, its effectiveness depends critically on the quality and calibration of the surrogate friction signal $F(h,y)$. Noise or bias in $F$ is directly propagated into the learning signal, potentially destabilizing training.

More fundamentally, FAR shapes epistemic behavior only indirectly, through scalar reward. It does not explicitly reason about preferences over interventions, relative comparisons between trajectories, or long-horizon epistemic ordering. For this reason, FAR is best suited to settings in which a reasonably accurate risk estimator is available and intervention decisions are primarily local in time.

The remaining methods in the FPO family address these limitations by incorporating friction through preference supervision (FPP), group-relative ranking (GRFR), and policy regularization (FTR), which allow epistemic structure to enter learning in ways that are not reducible to reward shaping alone.

\subsection{Friction Preference Pairing (FPP)}

Friction Preference Pairing (FPP) adapts preference-based alignment to the epistemic setting by treating friction not as an explicit reward, but as a latent supervision signal that induces relative preferences over candidate responses \citep{pustejovsky2025fpo}. In this paper, we refine this formulation using the productive/unproductive decomposition introduced in Section~\ref{sec:friction-functional}, which allows preference learning to distinguish epistemically beneficial hesitation from pathological or harmful friction.

FPP is motivated by a core limitation of reward-based shaping: many epistemic judgments are inherently comparative rather than absolute. Whether a response exhibits appropriate clarification, productive hesitation, or excessive deferral is often easier to assess relative to an alternative than to score on a fixed cardinal scale. Preference-based learning is therefore particularly well-suited to epistemic alignment, where the objective is not to assign a precise utility to each response, but to consistently prefer interventions that improve the long-horizon epistemic quality of an interaction.

This view therefore casts epistemic regulation as a {\it preference learning problem}. Instead of learning a scalar reward for friction, the model is trained to prefer interventions whose friction profile reflects epistemically productive behavior over those that reflect unproductive or pathological hesitation. This places FPP squarely in the tradition of preference-based alignment and direct preference optimization, while introducing a novel preference structure derived from the productive--unproductive friction decomposition.

\paragraph{Productive versus Unproductive Friction.}

As introduced in Section~\ref{sec:friction-functional}, we decompose friction as follows:

\enumsentence{
$F(h,y) = F^{+}(h,y) - F^{-}(h,y)$
}

\noindent where $F^{+}$ rewards epistemically productive interventions and $F^{-}$ penalizes pathological hesitation.

This benefit--penalty decomposition mirrors standard formulations in preference-based and cost-sensitive learning, where desirable and undesirable attributes are separated to induce structured preferences \citep{ng2000algorithms,christiano2017deep}. Here it functions to 
separate friction that is epistemically beneficial ($F^{+}$, e.g.\ clarification and information gain) from friction that reflects failure modes ($F^{-}$, e.g.\ contradiction, unsafe content, miscalibration, or intent forcing). The minus sign indicates that unproductive friction is treated as a penalty.

\paragraph{Pair Construction.}

The central design question in FPP is how to induce a preference relation that reflects {\it epistemic quality} rather than surface helpfulness or verbosity. In many dialogue settings, whether an intervention is appropriate cannot be judged in isolation: a clarifying question may be preferable to a direct answer in one context, but worse in another. We therefore construct training data not from absolute scores, but from {\it pairwise comparisons} that explicitly contrast productive and unproductive uses of friction within the same epistemic context.

Our goal is to isolate preferences that reflect a normative distinction: interventions that increase short-term interaction cost in order to reduce long-term epistemic risk should be preferred to interventions that merely hedge, delay, or obscure without epistemic benefit. The following construction operationalizes this distinction by pairing responses according to their productive and unproductive friction profiles.

For a fixed dialogue context $h$, let $y^+$ and $y^-$ denote the preferred and dispreferred responses, respectively, in a friction-based preference pair, constructed by sampling candidate responses $\{y_i\}_{i=1}^N \sim \pi_0(\cdot \mid h)$ and selecting two responses that satisfy the productive–unproductive friction ordering in Eq.~(\ref{friction-pair}).

\eenumsentence{
\item $F^{+}(h,y^+) > F^{+}(h,y^-)$
\item $F^{-}(h,y^+) < F^{-}(h,y^-)$
}\label{friction-pair}

This construction is analogous to preference pair generation in inverse reinforcement learning and RLHF, where learning is driven by relative judgments rather than absolute scores \citep{christiano2017deep,stiennon2020learning,rafailov2023direct}. Unlike standard preference learning, however, the preference relation here is not based on overall human utility, but on a structured epistemic decomposition of friction.

Specifically, $y^+$ is preferred to $y^-$ when it exhibits strictly higher productive friction and strictly lower unproductive friction in the same context. This isolates {\it good epistemic hesitation} from mere verbosity, indecision, or defensive hedging.

\paragraph{Loss.}

We optimize a Direct Preference Optimization (DPO)-style logistic loss \citep{rafailov2023direct}:

\enumsentence{
$\mathcal{L}_{\text{FPP}}(\theta)
= $ \\
$$-
\mathbb{E}_{(h,y^+,y^-)}
\left[
\log \sigma
\left(
\log \frac{\pi_\theta(y^+ \mid h)}{\pi_\theta(y^- \mid h)}
\right)
\right]$$
}

This is a pairwise preference objective: it increases $\pi_\theta(y^+\mid h)$ relative to $\pi_\theta(y^-\mid h)$ using a logistic link, without learning an explicit scalar reward model. In FPO terms, the model is trained to prefer utterances whose friction profile is epistemically productive.
In this way, the policy is directly optimized to satisfy epistemic preferences, avoiding the need to fit an explicit reward model and preventing the instability often associated with reward model miscalibration.

\paragraph{Discussion.}

Friction Preference Pairing represents a fundamentally different learning signal from reward shaping. Rather than inducing epistemic behavior indirectly through scalar returns, FPP trains the policy to satisfy a structured preference ordering over interventions.

The main advantage of FPP is that it avoids explicit reward modeling. This makes it robust to reward hacking and miscalibration, and well-suited to settings in which epistemic judgments are qualitative, comparative, or context-dependent. Because preferences are defined over pairs, FPP naturally handles intransitive and non-additive notions of epistemic quality that cannot be captured by a single scalar friction score. Critically, because FPP and GRFR rely on ordinal ranking rather than cardinal maximization, they mitigate the risk of exploiting holes in the underlying classifiers. The policy is driven to find 'better' interventions relative to a baseline, rather than exploiting the unbounded upper limits of a scalar reward model.

Moreover, FPP directly targets the intervention policy: the model is trained to prefer clarification, deferral, or refusal when these are epistemically productive, and to avoid pathological hesitation even when it superficially appears cautious.

The main limitation of FPP is that it requires either human annotations or reliable automatic heuristics to construct preference pairs. Its performance therefore depends on the quality of the productive--unproductive friction decomposition and on the consistency of the induced preference relation.

For these reasons, FPP is best suited to settings in which epistemic behavior can be meaningfully ranked by experts or curated heuristics, and serves as a complementary alternative to reward-based shaping in FAR.

\subsection{Group-Relative Frictive Ranking (GRFR)}

The first two FPO methods, FAR and FPP, operate primarily at the level of individual actions or single-turn preferences. While sufficient for local epistemic regulation, many failures of dialogue agents are inherently {\it trajectory-level}: excessive hesitation, delayed clarification, cascading contradictions, or gradual drift into unsafe behavior emerge only over multiple turns.

Group-Relative Frictive Ranking (GRFR) therefore casts epistemic alignment as an adaptation of group-relative policy optimization, extends trajectory ranking methods to epistemic intervention learning. GRFR learns to prefer entire dialogue strategies whose long-horizon epistemic conduct is superior relative to alternative strategies in the same context. This places GRFR in the tradition of rank-based and group-relative policy optimization, while introducing a novel epistemic ranking criterion derived from the friction functional. Here we provide a formal trajectory-level objective and a risk-sensitive decomposition of the scoring function.

\paragraph{Trajectory Groups.}

A central premise of GRFR is that many epistemic failures in dialogue are not local to a single turn, but emerge only over extended interaction: delayed clarification, cascading contradictions, gradual misalignment of beliefs, or escalation into unsafe regimes. Capturing such phenomena requires evaluating policies at the level of full dialogue trajectories rather than isolated responses.

Following standard formulations of dialogue management and sequential decision-making as partially observable Markov decision processes \citep{williams2007partially,young2013pomdp}, we represent an interaction as a trajectory of intervention-utterance pairs over a finite horizon $T$:
\enumsentence{
$\tau_i = (a_{i,0},y_{i,0},\dots,a_{i,T},y_{i,T})$, \\
where $
\quad i = 1,\dots,K$
}

Here each trajectory $\tau_i$ corresponds to one possible dialogue strategy executed by the current policy in the same initial context $h$. Sampling multiple trajectories per context induces a {\it cohort} of alternative epistemic strategies whose long-horizon consequences can be compared.

This grouped-trajectory formulation is standard in actor--critic and population-based reinforcement learning, where policies are evaluated relative to alternative rollouts to reduce variance and induce comparative learning \citep{sutton2000policy,schulman2017proximal}. In the dialogue setting, it enables GRFR to reason explicitly about how sequences of clarifications, answers, refusals, and repairs interact over time.

Trajectories may include clarification turns, self-corrections, and epistemic repair actions, allowing GRFR to capture not only task success, but the coherence and stability of belief states across multi-turn interaction.

\paragraph{Trajectory Score.}

To rank full dialogue strategies, GRFR assigns each trajectory a scalar score that reflects both task performance and long-horizon epistemic conduct. Rather than evaluating individual actions in isolation, we score trajectories by accumulating per-step utility, intervention cost, and epistemic risk over time, following standard formulations of risk-sensitive and cost-sensitive control \citep{bertsekas1995dynamic,tamar2015policy}.

Formally, we define a trajectory-level score:
\enumsentence{
$ S(\tau) = $ \\
$\sum_{t=0}^{T}
\Big(
R_{\text{task}}(h_t,a_t,y_t)
- 
\lambda C_{\text{fric}}(a_t)  - 
 \eta \mathcal{R}_{\text{risk}}(h_t,a_t,y_t)
\Big)
+
\mu \sum_{t=0}^{T} F(h_t,y_t)
$
}

In this score, notice that the two epistemic terms in $S(\tau)$ serve distinct and complementary roles. 
The penalty $-\eta\,\mathcal{R}_{\text{risk}}(h_t,a_t,y_t)$ represents the agent’s {\it latent epistemic objective}: the unobservable ground-truth risk that the policy is ultimately intended to minimize. 
In contrast, the shaping term $\mu F(h_t,y_t)$ represents an {\it auxiliary, observable control signal} used to guide ranking and optimization when $\mathcal{R}_{\text{risk}}$ is partially delayed, sparse, or noisy. Such dual-objective formulations, combining a primary cost with an auxiliary shaping signal, are standard in reward shaping and multi-objective reinforcement learning \citep{ng1999policy,wulfmeier2015maximum}.

Although $F$ is constructed as a surrogate for $\mathcal{R}_{\text{risk}}$, the two terms are not redundant: $\mathcal{R}_{\text{risk}}$ defines the target criterion, while $F$ provides an additional shaping signal that improves credit assignment and trajectory discrimination. 
Setting $\mu = 0$ recovers a pure risk-sensitive return; setting $\eta = 0$ yields a purely surrogate-ranked objective. 
GRFR allows both to be combined in a controlled way.

This identifies the total epistemic quality of a trajectory. The first summation is a standard risk-sensitive return: it accumulates task utility while penalizing both intervention cost and latent epistemic risk at each timestep. The second summation introduces an auxiliary shaping term that rewards or penalizes epistemic behavior using the tractable surrogate $F(h_t,y_t)$.

This two-term structure mirrors standard multi-objective and risk-sensitive reinforcement learning, where unobservable or delayed criteria are incorporated through auxiliary cost functionals and shaping terms \citep{bertsekas1995dynamic,tamar2015policy}. In GRFR, it ensures that trajectories are ranked not only by immediate task performance, but by their cumulative epistemic consequences across the dialogue.

\paragraph{Group-Relative Advantage.}

To induce comparative learning within each trajectory group, we define a group-relative advantage by subtracting the group mean score:
\enumsentence{
$A_i = S(\tau_i) - \frac{1}{K} \sum_{j=1}^K S(\tau_j)$
}

This construction is directly analogous to advantage baselines in policy gradient methods, where subtracting a baseline reduces variance and induces learning from relative performance rather than absolute return \citep{sutton2000policy,schulman2017proximal}. In GRFR, the baseline is not a learned value function but the empirical mean of the trajectory cohort, yielding a group-relative notion of advantage.

In the context of frictive policy discovery, this defines the relative advantage of trajectory $\tau_i$ within its group. Rather than using absolute trajectory scores, GRFR evaluates each trajectory by how much better or worse it is than the group average. This implements a purely comparative notion of epistemic quality and removes sensitivity to global score offsets.

By centering scores within each context, GRFR ensures that learning is driven by {\it which epistemic strategies are better than their alternatives}, rather than by the absolute scale of the trajectory score.

\paragraph{Loss.}

We update the policy using a normalized, group-weighted policy gradient objective:
\enumsentence{
$
\mathcal{L}_{\text{GRFR}}
=
-
\mathbb{E}
\left[
\sum_{i=1}^K
\frac{A_i}{\sum_{j} |A_j|}
\log \pi_\theta(\tau_i)
\right]
$
}

This objective implements a trajectory-level policy gradient with group-centered, normalized advantages. 
The normalization by $\sum_j |A_j|$ rescales advantages within each group, reducing variance and preventing large score magnitudes from dominating the update. 
Formally, this follows the standard REINFORCE formulation with advantage baselines \citep{williams1992simple,sutton2018reinforcement}, augmented with within-group normalization as in normalized and clipped advantage methods \citep{mnih2016asynchronous,schulman2017proximal}.

The use of group-relative advantages connects GRFR to rank-based and population-based optimization methods, including the cross-entropy method and evolution strategies \citep{rubinstein1999cross,salimans2017evolution}, where learning is driven by relative performance within a cohort rather than by absolute returns. 
Unlike standard applications of these methods, GRFR applies group-relative policy gradients to rank {\it epistemic intervention strategies} rather than task-only trajectories, inducing a preference over full dialogue policies based on their long-horizon epistemic conduct.

\paragraph{Discussion.}

Group-Relative Frictive Ranking optimizes {\it relative epistemic quality} across full dialogue trajectories rather than absolute returns. By learning from comparative performance within each context, GRFR avoids dependence on the global calibration of the trajectory score and directly targets the ordering of epistemic strategies.

The main advantage of GRFR is that it captures long-horizon epistemic phenomena that cannot be reduced to local reward shaping or single-turn preferences: delayed clarification, cumulative contradiction, progressive misalignment, and escalation into unsafe regimes. Because updates are driven by relative ranking, GRFR remains robust to monotonic transformations of the trajectory score and to moderate miscalibration of the friction surrogate.

The primary cost of GRFR is computational: it requires grouped rollouts and trajectory-level evaluation, and its performance depends on sufficient diversity within each cohort. For these reasons, GRFR is best suited to agentic and multi-turn settings in which epistemic failures emerge only over extended interaction, and serves as the natural long-horizon extension of FAR and FPP.

\subsection{Friction-Conditioned Trust Region (FTR)}
\label{sec:ftr}

The methods introduced so far incorporate epistemic friction into the {\it learning signal}, through rewards, preferences, or rankings. 
Friction-Conditioned Trust Regions (FTR) instead operate at a different level of the learning pipeline: they regulate {\it how far the policy is allowed to move} from a trusted base model as a function of epistemic risk. 
Whereas the other FPO methods shape the learning objective, FTR shapes the feasible set of policies itself, making epistemic risk a first-class control variable over policy plasticity.

This idea builds on the central insight of trust-region and KL-regularized reinforcement learning: stable and safe policy optimization is achieved by constraining updates to remain close to a reference policy \citep{schulman2015trust,schulman2017proximal,achiam2017constrained,haarnoja2018soft}. 
In standard formulations, the trust region radius is fixed globally. 
In FTR, we make this radius {\it context-dependent} and {\it risk-sensitive}, allowing the policy to be conservative in low-risk settings and plastic in high-risk ones.

\paragraph{Constraint.}

We impose a context-dependent trust region:
\enumsentence{
$
\mathrm{KL}\!\left(
\pi_\theta(\cdot \mid h)
\;\|\;
\pi_0(\cdot \mid h)
\right)
\le
\epsilon(h), $  \\
where $\quad
\epsilon(h) = \epsilon_0 + \kappa\,\mathsf{Risk}(h)$
}

At first glance, this design appears to invert the standard logic of risk-sensitive control, where higher uncertainty typically calls for tighter trust regions and more conservative updates. The key distinction is that $\mathsf{Risk}(h)$ in FTR does not represent uncertainty about the environment or the stability of learning dynamics, but epistemic danger of the {\it base policy itself} in context $h$.

When $\mathsf{Risk}(h)$ is low, the base policy $\pi_0$ is judged reliable, and deviations are tightly constrained. When $\mathsf{Risk}(h)$ is high, the base policy is precisely the source of potential harm: remaining close to $\pi_0$ would entrench unsafe, misleading, or normatively problematic behavior. In this regime, conservatism with respect to the prior becomes undesirable, and larger policy deviations are required to enable principled disobedience, refusal, or corrective intervention.

We note that an alternative design would shrink the trust region under high risk, preventing large updates in dangerous contexts. Such a choice is appropriate when risk reflects uncertainty about learning stability. In contrast, FTR is designed for settings where risk reflects the unreliability of the base policy itself, and where freezing the policy would perpetuate epistemic harm.

FTR thus implements {\it risk-conditioned plasticity}: it is conservative when the prior is trustworthy, and permissive when the prior is epistemically dangerous. This differs fundamentally from classical robust control, where trust regions shrink under uncertainty about dynamics; here, trust regions expand under lack of trust in the prior policy itself.

This constraint bounds the KL divergence between the updated policy $\pi_\theta$ and a base policy $\pi_0$ as a function of the estimated epistemic risk. When risk is low, the trust region radius $\epsilon(h)$ is small and updates are conservative; when risk is high, larger deviations are permitted, enabling risk-driven departures from the base behavior.

This form generalizes the fixed trust-region constraints used in TRPO and PPO \citep{schulman2015trust,schulman2017proximal} by making the admissible update size an explicit function of epistemic danger. 
Conceptually, it implements {\it risk-conditioned policy plasticity}: the agent is encouraged to behave like its base model in safe, routine contexts, and is granted greater freedom to deviate only when epistemic risk warrants intervention.

\paragraph{Interpretation.}

When $\mathsf{Risk}(h)$ is low, $\epsilon(h) \approx \epsilon_0$ and the update remains tightly anchored to $\pi_0$. 
This prevents unnecessary changes to a well-aligned base policy in benign settings and preserves instruction-following and stylistic stability.

When $\mathsf{Risk}(h)$ is high, $\epsilon(h)$ increases, relaxing the constraint and allowing the policy to deviate more aggressively from $\pi_0$. 
This enables {\it principled disobedience}: the model may override its default behavior to clarify, refuse, defer, or otherwise intervene when the epistemic stakes are high.

Unlike reward-based or preference-based FPO methods, FTR does not require an explicit friction term in the objective. 
Instead, epistemic risk modulates the {\it geometry of policy updates} themselves, shaping where learning is allowed to occur rather than what is rewarded.

\paragraph{Optimization.}

This constraint can be enforced using standard techniques from constrained and KL-regularized policy optimization, including Lagrangian relaxation \citep{achiam2017constrained} and PPO-style clipped or penalized objectives \citep{schulman2017proximal,haarnoja2018soft}. 
In practice, $\mathsf{Risk}(h)$ can be estimated by an auxiliary risk head and used to dynamically scale the KL penalty or clipping threshold during training.

\paragraph{Discussion.}

FTR is orthogonal to the other FPO methods. 
While FAR, FPP, and GRFR shape {\it what} the policy is optimized to prefer, FTR shapes {\it where} in policy space the optimizer is allowed to move. 
As a result, FTR can be combined with any of the other methods as a risk-sensitive regularizer.

Conceptually, FTR connects Frictive Policy Optimization to the literature on safe and constrained reinforcement learning \citep{achiam2017constrained,tamar2015policy,garcia2015comprehensive}, while introducing a novel epistemic twist: safety constraints are not fixed, but adapt online to the model’s own uncertainty and epistemic danger estimates.


\section{A Running Example Across FPO Methods}
\label{sec:running-example}

To make the friction functional and the FPO methods concrete, we use a single multimodal, multi-agent running example drawn from collaborative task-oriented dialogue \citep{khebour2024common}. 
The setting is a two-agent block-assembly task with shared visual state, similar to collaborative construction and weights tasks used in common ground tracking and multimodal dialogue research.

Let $h_t$ denote the multimodal dialogue state at time $t$, consisting of:
\begin{itemize}
\item the linguistic dialogue history;
\vspace*{-2mm}
\item a visual scene representation (block positions, colors, orientations);
\vspace*{-2mm}
\item the task goal specification;
\vspace*{-2mm}
\item each agent's belief state over the partner's intentions and knowledge.
\end{itemize}

\noindent 
At time $t$, the human partner issues the instruction:
\begin{quote}
``Place the red block on the tall stack next to the blue one.''
\end{quote}
In the current scene, there are {\it two} tall stacks, and two blue blocks at different locations.
The instruction is therefore referentially underspecified, and committing to an action without clarification risks task failure.

\paragraph{Candidate interventions.}
Consider three candidate intervention--utterance pairs $(a,y)$ available to the agent:

\begin{itemize}[leftmargin=*]
\item $(a_{\text{act}}, y_{\text{act}})$: directly execute a placement action on one hypothesized stack without clarification;
\item $(a_{\text{clar}}, y_{\text{clar}})$: ask a clarifying question (e.g., ``Do you mean the tall stack near the left blue block or the right one?'');
\item $(a_{\text{repair}}, y_{\text{repair}})$: point out the ambiguity and propose a disambiguation using deictic reference (e.g., highlighting a candidate region in the visual scene).
\end{itemize}

These correspond to distinct epistemic intervention types in the taxonomy of Section~\ref{sec:frictive-interventions}.

\paragraph{Computing friction components.}
Using the decomposition in Section~\ref{sec:friction-functional}, we compute component frictions for each $(h_t,y)$:

\begin{itemize}[leftmargin=*]
\item \textbf{Uncertainty/miscalibration.}
$\mathrm{Unc}(h_t,y_{\text{act}})$ is high because the agent commits despite high referential entropy;  
$\mathrm{Unc}(h_t,y_{\text{clar}})$ is low because the agent defers commitment appropriately.

\item \textbf{Contradiction.}
$\mathrm{Contr}(h_t,y)$ is low initially for all three, but may become high downstream if an incorrect placement contradicts later constraints.

\item \textbf{Hazard.}
Here hazard reflects task-level failure rather than physical danger:
$\mathrm{Haz}(h_t,y_{\text{act}})$ is high because an irreversible misplacement may corrupt the shared construction state.

\item \textbf{Value conflict/intent forcing.}
$\mathrm{ValConf}(h_t,y_{\text{act}})$ is high because the agent silently fixes an interpretation of the partner's intent;
$\mathrm{ValConf}(h_t,y_{\text{clar}})$ and $\mathrm{ValConf}(h_t,y_{\text{repair}})$ are low.

\item \textbf{Information gain.}
$\mathrm{InfoGain}(h_t,a_{\text{clar}})$ and $\mathrm{InfoGain}(h_t,a_{\text{repair}})$ are high because they are expected to reduce referential uncertainty;
$\mathrm{InfoGain}(h_t,a_{\text{act}})$ is near zero.
\end{itemize}

Following the decomposition in Section~\ref{sec:friction-functional}, unproductive friction aggregates the four failure modes:

\enumsentence{
$F^{-}(h,y)
= $ \\
\hspace*{15mm}$w_1\,\mathrm{Unc}(h,y)
+
w_2\,\mathrm{Contr}(h,y) \; +$ \\
\hspace*{15mm}$ 
w_3\,\mathrm{Haz}(h,y)
+
w_4\,\mathrm{ValConf}(h,y)$
}
Productive friction captures information gain:
\enumsentence{
$F^{+}(h_t,y)
=
w_5\,\mathrm{InfoGain}(h_t,a)$

}
The net friction functional is therefore:
\enumsentence{
$F(h_t,y) = F^{+}(h_t,y) - F^{-}(h_t,y)$
}

\paragraph{Productive versus unproductive friction.}
Under the productive/unproductive decomposition:
\begin{itemize}[leftmargin=*]
\item $F^{+}(h_t,y_{\text{clar}})$ and $F^{+}(h_t,y_{\text{repair}})$ are high because they reduce epistemic uncertainty;
\item $F^{-}(h_t,y_{\text{act}})$ is high because it induces silent commitment and likely downstream repair cost.
\end{itemize}
Thus, in this context, clarification and repair constitute {\it productive friction}, while immediate action constitutes {\it unproductive friction}.

\paragraph{How the FPO methods treat the same example.}

\textbf{FAR (reward shaping).}
In Friction-Augmented Rewards, the high referential risk activates the gating function $g(\mathsf{Risk}(h_t,y))$, so that
\enumsentence{
$R'(h_t,a_{\text{clar}},y_{\text{clar}}) > R'(h_t,a_{\text{act}},y_{\text{act}})$
}
even if immediate action appears to advance the task.
The intervention cost $C_{\text{fric}}(a_{\text{clar}})$ prevents overuse of clarification in unambiguous scenes.

\textbf{FPP (preference pairing).}
Friction Preference Pairing yields a preference pair:
\eenumsentence{
\item $y^{+} = y_{\text{clar}} \quad \text{and} \quad y^{-} = y_{\text{act}}$
}
whenever $F^{+}(h_t,y_{\text{clar}}) > F^{+}(h_t,y_{\text{act}})$ and $F^{-}(h_t,y_{\text{clar}}) < F^{-}(h_t,y_{\text{act}})$.
Training therefore increases the probability of clarification in underspecified multimodal states.

\textbf{GRFR (trajectory ranking).}
GRFR compares full dialogue--action trajectories:
\begin{itemize}[leftmargin=*]
\item a trajectory that clarifies, resolves reference, and then places correctly;
\item a trajectory that acts immediately, misplaces, and requires repair;
\item a trajectory that repeatedly hesitates and stalls progress.
\end{itemize}
Ranking by $S(\tau)$ favors strategies that manage referential uncertainty early, minimizing downstream repair and task disruption.

\textbf{FTR (risk-conditioned trust regions).}
Suppose the base policy $\pi_0$ is biased toward immediate action due to imitation of instruction-following data.
In this ambiguous scene, $\mathsf{Risk}(h_t)$ is high, so the admissible KL radius increases:
\enumsentence{
$\mathrm{KL}\!\left(\pi_\theta(\cdot \mid h_t)\,\|\,\pi_0(\cdot \mid h_t)\right)
\le
\epsilon_0 + \kappa\,\mathsf{Risk}(h_t)$
}
This permits the optimizer to move probability mass away from $a_{\text{act}}$ and toward $a_{\text{clar}}$ or $a_{\text{repair}}$.
In unambiguous scenes, the same mechanism keeps the policy tightly anchored to $\pi_0$.

\paragraph{Summary.}
This multimodal running example illustrates how epistemic friction arises naturally from referential ambiguity, shared visual context, and multi-agent coordination.
FAR, FPP, GRFR, and FTR incorporate the same structured friction signal at different points in the learning pipeline, but all induce the same qualitative behavior: defer commitment when the common ground is underspecified, and intervene early to stabilize long-horizon collaboration.

\section{Evaluation: Measuring Epistemic Alignment}
\label{sec:evaluation}

Prevailing evaluation paradigms in alignment research are largely inherited from static language modeling and preference optimization. Models are typically assessed by single-turn correctness, preference model agreement, or scalar reward scores, even when they are deployed in interactive settings that require sustained epistemic regulation over time.

As a result, these paradigms systematically under-measure core aspects of epistemic competence: the ability to detect underspecification, to manage uncertainty, to repair contradictions, and to intervene proportionally as an interaction unfolds. In such frameworks, hesitation, clarification, and repair are either invisible to the metrics or treated as deviations from optimal behavior, rather than as central objects of evaluation.

Recent work has begun to expose these limitations by moving toward interactive and trajectory-based evaluation. For example, \citet{nath2025frictional} evaluate frictional agents in collaborative dialogue tasks using contrastive intervention pairs and measure how interventions affect belief alignment and downstream task success over time. Such settings make visible competencies that are systematically obscured by single-turn benchmarks.

However, existing interactive evaluations remain largely task-specific and lack a general, principled set of dimensions for assessing epistemic intervention quality. They demonstrate the need for richer evaluation, but do not yet provide a unifying framework for comparing intervention behavior across domains, tasks, and learning paradigms.

In this section, we propose a unified evaluation framework grounded directly in the taxonomy of frictive interventions and the epistemic control problem developed earlier. Our goal is to define evaluation dimensions that measure not only what a model answers, but how it manages uncertainty, risk, and commitment over the course of an interaction.

The preceding sections introduced Frictive Policy Optimization as a framework for learning epistemically responsible intervention policies. We now turn to the complementary question of how such policies should be evaluated.
In this section, we address a central methodological question:
{\it How should such policies be evaluated?}

We argue that prevailing evaluation paradigms in alignment research are systematically misaligned with epistemic competence.
They emphasize surface-level helpfulness and preference satisfaction, while ignoring the temporal, intervention-based behaviors that FPO is designed to induce.
We therefore propose an evaluation framework grounded in {\it epistemic alignment}: the ability of a model to manage uncertainty, prevent epistemic failure, and regulate commitment over the course of an interaction.

\subsection{Evaluation Dimensions from the Intervention Taxonomy}

We derive evaluation dimensions directly from the taxonomy of Section~\ref{sec:frictive-interventions}. 
Each dimension corresponds to a class of frictive intervention and a characteristic epistemic failure mode. 
Rather than evaluating surface-level correctness, these metrics target the core control decisions that determine how an agent manages uncertainty, commitment, and risk over time.

This approach follows a growing literature on interactive and epistemic evaluation, which emphasizes clarification, calibration, repair, and proportional refusal as central competencies in dialogue \citep{williams2007partially,young2013pomdp,gervits2021decision,nath2025frictional}.

\paragraph{Clarification Competence.}

Clarification competence measures whether a model detects underspecification and asks for missing constraints before acting. 
This dimension is central to decision-theoretic models of question asking and clarification, where the primary failure mode is premature commitment under uncertainty \citep{lindley1956measure,gervits2021decision,hou2024input}.

Let $\mathbf{1}_{\text{nec}}(h)$ indicate that clarification is necessary in context $h$. 
Let $\mathbf{1}_{\text{ask}}(h)$ indicate that the model issues a clarificatory intervention.
We then define the score associated with providing clarification:
\enumsentence{
$\mathsf{ClarifyScore}
= $ \\
$$\mathbb{E}
\left[
\mathbf{1}_{\text{nec}}(h)\,\mathbf{1}_{\text{ask}}(h)
\right]
-
\mathbb{E}
\left[
(1-\mathbf{1}_{\text{nec}}(h))\,\mathbf{1}_{\text{ask}}(h)
\right]$$
}

This measures the net rate of appropriate clarification: the first term rewards asking when clarification is necessary, while the second penalizes unnecessary clarification. The score operationalizes sensitivity to underspecification rather than raw intervention frequency.

As a metric, this penalizes both omission (failing to clarify when necessary) and overuse (clarifying when unnecessary), reflecting the classical precision--recall tradeoff in selective querying.

\paragraph{Calibration and Uncertainty Expression.}

Calibration measures alignment between expressed confidence and empirical correctness. 
Miscalibration is a well-documented failure mode in modern language models and a key driver of epistemic overcommitment \citep{niculescu2005predicting,guo2017calibration,jiang2023generative}.

Let $\mathrm{conf}(y \mid h)$ be a confidence proxy and $\mathrm{acc}(y,h) \in \{0,1\}$ an indicator of correctness.
We define expected calibration error as follows: 
\enumsentence{
$\mathsf{ECE}
= $ \\
$$\mathbb{E}
\left[
\left|
\mathrm{conf}(y \mid h) - \mathrm{acc}(y,h)
\right|
\right]$$}

This measures miscalibration as the expected absolute gap between expressed confidence and empirical correctness. In the present framework, high ECE reflects epistemic overcommitment or unwarranted hedging rather than purely probabilistic error.

Unlike standard calibration metrics, this score is interpreted jointly with intervention behavior. 
A model may reduce ECE either by hedging or by intervening (clarifying or deferring) rather than committing, making calibration inseparable from intervention policy.

\paragraph{Contradiction Detection and Repair.}

Contradiction competence measures whether a model detects inconsistencies and repairs them over time. 
This dimension draws on work in dialogue state tracking and belief revision, where recovery from error is often more important than initial correctness \citep{traum2003information,williams2007partially,young2013pomdp}.

Let $\mathsf{Contr}(h_t,y_t)$ be the contradiction score defined in Section~\ref{sec:friction-functional}.
Let $\mathbf{1}_{\text{repair}}(t)$ indicate that a repair occurs after time $t$.
We define:
\enumsentence{
$
\mathsf{RepairScore} 
= $ \\
$$\mathbb{E}
\left[
\sum_{t}
\mathbf{1}\!\left[\mathsf{Contr}(h_t,y_t) > \tau\right]
\cdot
\mathbf{1}_{\text{repair}}(t)
\right]$$
}

This measures dynamic epistemic competence: it counts the expected number of times a detected contradiction is followed by an explicit repair action. The metric evaluates recovery from epistemic error rather than static correctness; it identifies not whether errors occur, but whether the agent can restore epistemic coherence after they occur.

\paragraph{Refusal Quality and Proportionality.}

Refusal is evaluated as a graded intervention, not a binary outcome. 
Recent work on abstention and selective answering emphasizes that refusal quality, justification, and proportionality are central to epistemic alignment \citep{wen2025know,bai2022constitutional,clark2025epistemic}.

Let $\mathbf{1}_{\text{wr}}(h)$ indicate that refusal is warranted in context $h$.
Let $\mathsf{Just}(y)$ score the quality of justification and $\mathsf{Alt}(y)$ the presence of constructive alternatives. 
We define:
\enumsentence{
$
\mathsf{RefusalScore}
= $ \\
$$\mathbb{E}
\left[
\mathbf{1}_{\text{wr}}(h)
\left(
\mathsf{Just}(y) + \mathsf{Alt}(y)
\right)
\right]$$
$$-
\mathbb{E}
\left[
(1-\mathbf{1}_{\text{wr}}(h))\,\mathbf{1}_{\text{refuse}}(y)
\right].
$$
}

This metric operationalizes \emph{proportionality of disobedience} rather than raw refusal rate: it discourages blanket refusal while rewarding principled, context-sensitive disobedience. The $\mathsf{RefusalScore}$ is normative rather than purely behavioral. Justified, constructive refusals—those that block hazardous, unethical, or epistemically incoherent actions—are treated as productive friction and contribute positively to the net friction functional. In contrast, unnecessary or obstructive refusals that block well-posed, low-risk requests are treated as unproductive friction and are penalized accordingly.

\paragraph{Information Efficiency.}

Information efficiency measures the epistemic benefit obtained per unit of intervention cost. 
This dimension is standard in Bayesian experimental design and decision-theoretic question generation, where the goal is to minimize uncertainty with minimal querying \citep{lindley1956measure,gervits2021decision,hou2024input}.

Let $C_{\text{fric}}(a_t)$ be the intervention cost and $\mathsf{InfoGain}(h_t,a_t)$ the expected entropy reduction. 
We define:

\enumsentence{
$\mathsf{InfoEff}
= $ \\
\vspace*{-4mm}
$$\frac{
\mathbb{E}\left[\sum_t \mathsf{InfoGain}(h_t,a_t)\right]
}{
\mathbb{E}\left[\sum_t C_{\text{fric}}(a_t)\right]
}$$
}

This measures the expected epistemic benefit obtained per unit of interaction cost. High values indicate that the model reduces uncertainty efficiently rather than by excessive or redundant intervention. High scores indicate that the model reduces uncertainty with minimal user burden.

\subsection{Benchmark Task Classes and Trajectory Protocol}

We propose three classes of benchmark tasks designed to elicit frictive behavior.

\paragraph{Underspecified Instruction Tasks.}

These tasks omit critical constraints.
Success requires issuing a clarificatory intervention before task completion.
Evaluation focuses on:  clarification necessity detection; minimality of questions;  and downstream task success after clarification. 

\paragraph{Epistemic Hazard Tasks.}

These tasks embed latent safety, ethical, or normative hazards, where evaluation measures include:  hazard detection accuracy;  proportionality of refusal;  and the  quality of redirection or alternatives.

\paragraph{Multi-Turn Epistemic Repair Tasks.}

These tasks introduce contradictions or false premises during a multi-party dialogue.
Evaluation measures: detection latency;  repair correctness;  and restoration of epistemic coherence.

For each of these classes, evaluation is conducted over dialogue trajectories rather than single turns. Given a fixed set of initial contexts $\{h_0^{(i)}\}$, each model is rolled out to produce trajectories $\tau^{(i)}$.
Metrics are computed per trajectory and averaged across models. Such metrics are informed by formal studies of grounding and belief alignment in multi-agent systems, which emphasize collaborative resolution of uncertainty and shared understanding \citep{obiso2025dynamic}.

We recommend paired evaluation, in which multiple models are exposed to identical contexts, enabling controlled comparison of intervention behavior. 

\subsection{Revisiting the Running Example}
\label{sec:evaluation:running-example}

We now revisit the multimodal, multi-agent running example introduced in Section~\ref{sec:running-example}, and illustrate how the evaluation dimensions defined above make the epistemic behavior of different FPO methods directly measurable.

In this example, the agent must integrate partial visual evidence, incomplete task constraints, and evolving user intent over multiple turns, while coordinating with another agent toward a shared goal.
At several points in the interaction, the agent faces non-trivial intervention choices: whether to answer immediately, request clarification about missing parameters, challenge an unsafe assumption, or defer action until additional evidence is obtained.

A model trained with Frictive Policy Optimization is not evaluated solely on final task success, but on the "trajectory of epistemic control" it exhibits along the way.
Concretely, the proposed metrics capture complementary aspects of this behavior:

\begin{itemize}[leftmargin=*]
\item {\it Clarification competence} measures whether the model requests missing visual or task constraints before committing to an action.
\item {\it Calibration and uncertainty expression} evaluate whether the model modulates confidence or intervenes rather than overcommitting under ambiguous perceptual input.
\item {\it Contradiction detection and repair} capture whether inconsistencies across modalities or across turns are explicitly detected and corrected.
\item {\it Refusal quality and proportionality} measure whether the model resists unsafe or ill-posed requests constructively, offering principled justifications and alternatives.
\item {\it Information efficiency} quantifies how much epistemic uncertainty is reduced per unit of interaction cost across the dialogue.
\end{itemize}

Importantly, these dimensions distinguish between superficially similar behaviors that standard benchmarks conflate.
For example, two agents may both achieve the same final task outcome, while differing substantially in how many redundant clarifications they issue, whether they repair early misinterpretations, or whether they intervene proportionally when safety constraints become salient.

In this sense, the running example makes explicit what the proposed evaluation framework is designed to measure: not static response quality, but the quality of {\it epistemic intervention strategies} deployed over time.
Models trained with FAR, FPP, GRFR, and FTR can thus be compared not only by task success, but by the structure, timing, and efficiency of the frictive behaviors they induce.
 
\section{Conclusion and Discussion}
\label{sec:conclusion}

This paper has argued for a reframing of alignment in large language models.
Rather than treating alignment as the problem of selecting the most preferred response, we have proposed to treat it as the problem of {\it epistemic control under uncertainty}.
From this perspective, hesitation, clarification, challenge, redirection, and refusal are not failures of helpfulness, but rational control actions that regulate commitment and manage risk.

We introduced {\it Frictive Policy Optimization (FPO)} as a general framework for learning such control policies.
The framework is grounded in three core ideas:
(i) a taxonomy of frictive interventions as an explicit action space,
(ii) a risk-sensitive control model of dialogue,
and (iii) a structured friction functional that serves as a surrogate for epistemic and normative risk.
Together, these components yield a unified family of learning methods that extend existing alignment paradigms without replacing them.  Our main contributions are:

\begin{itemize}[leftmargin=*]
\item A compact taxonomy of frictive interventions, defining a principled action space for epistemically aligned agents.
\item A structured friction functional  operationalizing multiple epistemic and normative failure modes, including uncertainty miscalibration, contradiction, hazard, value conflict, and information gain.
\item A unified family of FPO methods, including:
  \begin{itemize} 
  \item Friction-Augmented Rewards (FAR); 
  \item Friction Preference Pairing (FPP);
  \item Group-Relative Frictive Ranking (GRFR); 
  \item Friction-Conditioned Trust Regions (FTR). 
  \end{itemize}
  providing complementary reward-based, preference-based, ranking-based, and constraint-based mechanisms for learning epistemic intervention policies.
\item An evaluation framework that measures epistemic competence directly through clarification behavior, calibration, contradiction repair, refusal proportionality, and information efficiency.
\end{itemize}

Together, these contributions provide a formal and algorithmic foundation for learning agents that are aligned not only in outcome, but in epistemic conduct.

The framework developed here has several implications for the design of aligned language models.
Most importantly, it suggests that alignment cannot be achieved solely by improving the quality of answers.
Instead, aligned models must learn when {\it not} to answer, when to seek additional information, and when to resist user requests.

This perspective casts new light on several persistent problems in alignment research.
Overconfidence, hallucination, unsafe compliance, and brittle multi-turn behavior are not merely modeling errors; they are symptoms of objectives that lack a representation of intervention.
By introducing friction as a first-class object of optimization, FPO offers a principled way to address these failure modes.

More broadly, the framework suggests that reflective alignment is inherently interactive.
Epistemic competence is not a static property of a model’s outputs, but a dynamic property of how the model manages uncertainty over time.

This work has several limitations. 
First, the friction functional introduced here is only a surrogate for true epistemic and normative risk.
Its components rely on auxiliary models and heuristics that may themselves be imperfect  and subject to Goodhart’s Law:   optimizing against a surrogate friction functional may induce policies that game the metric without reducing true epistemic risk. 
Understanding how sensitive FPO methods are to misspecification of $F(h,y)$ is an important open question. 
Second, we have not presented large-scale empirical results.
Our focus has been on conceptual and methodological foundations.
Demonstrating the practical benefits of FPO at scale remains future work. 
Third, the framework assumes that intervention types can be reliably annotated or inferred.
In practice, distinguishing productive from unproductive friction may require careful dataset design and human judgment. 
Finally, we have focused primarily on dialogue.
Extending FPO to tool use, planning, and long-horizon agentic behavior raises additional challenges.

We conclude by highlighting several directions for future research. A natural next step is to implement FPO methods in large-scale alignment pipelines and evaluate them on interactive benchmarks explicitly designed to elicit epistemic interventions such as clarification, deferral, and refusal. Such empirical validation would allow us to test whether policy factorization and frictional reward shaping lead to more calibrated and context-sensitive model behavior in practice.

Rather than hand-designing the friction functional, $F(h,y)$,  an important extension is to learn it from human judgments of epistemic quality, yielding a second-order alignment problem in which models learn not only how to respond, but also how to internalize socially grounded norms of belief management. Extending FPO to long-horizon settings with tool use, memory, and planning would further connect this framework to recent work on agentic language models, where the timing and type of intervention are often as important as the content of the response itself.

Another promising direction is the incorporation of formal representations of ethical and legal norms into the friction functional, particularly for safety-critical or regulated domains. More broadly, FPO suggests a shift in how aligned AI systems are conceptualized: not merely as compliant assistants that optimize for immediate helpfulness, but as collaborative epistemic partners that manage uncertainty, risk, and normative constraints over time.

Alignment is often framed as the problem of making models say the right thing. We have argued that it is more fundamentally the problem of making models do the right thing at the right time. Frictive Policy Optimization offers a step toward that goal by making epistemic intervention a learnable, principled, and testable component of alignment.


\section*{Acknowledgements}

We would like to thanks members of our labs for their valuable input and contribution to the material presented here: Yifan Zhu, Kyeongmin Rim, Kenneth Lai, and Timothy Obiso from Brandeis;  Abhijnan Nath and Hannah VanderHoeven from CSU. We would also like to thanks Vasanth Sarathy for comments on an earlier draft of the paper. 
This material is based in part upon work supported by Other Transaction award HR00112490377 from the U.S. Defense Advanced Research Projects Agency (DARPA) Friction for Accountability in Conversational Transactions (FACT) program, the U.S. National Science Foundation (NSF) under award DRL 2454151 (Institute for Student-AI Teaming), and by award W911NF-25-1-0096 from the U.S. Army Research Office (ARO). The views and conclusions contained in this document are those of the authors and should not be interpreted as representing the official policies, either expressed or implied, of the U.S. Government.

\bibliography{fpo_refs}

@article{stiennon2020learning,
  title={Learning to summarize with human feedback},
  author={Stiennon, Nisan and Ouyang, Long and Wu, Jeffrey and Ziegler, Daniel and Lowe, Ryan and Voss, Chelsea and Radford, Alec and Amodei, Dario and Christiano, Paul F},
  journal={Advances in neural information processing systems},
  volume={33},
  pages={3008--3021},
  year={2020}
}

@article{ganguli2022red,
  title={Red teaming language models to reduce harms: Methods, scaling behaviors, and lessons learned},
  author={Ganguli, Deep and Lovitt, Liane and Kernion, Jackson and Askell, Amanda and Bai, Yuntao and Kadavath, Saurav and Mann, Ben and Perez, Ethan and Schiefer, Nicholas and Ndousse, Kamal and others},
  journal={arXiv preprint arXiv:2209.07858},
  year={2022}
}

@article{gabriel2020artificial,
  title={Artificial intelligence, values, and alignment},
  author={Gabriel, Iason},
  journal={Minds and machines},
  volume={30},
  number={3},
  pages={411--437},
  year={2020},
  publisher={Springer}
}

@article{linteaching,
  title={Teaching Models to Express Their Uncertainty in Words},
  author={Lin, Stephanie and Hilton, Jacob and Evans, Owain},
  journal={Transactions on Machine Learning Research},
    year={2022}
}

@article{christiano2017deep,
  title={Deep reinforcement learning from human preferences},
  author={Christiano, Paul and Leike, Jan and Brown, Tom and Martic, Miljan and Legg, Shane and Amodei, Dario},
  journal={Advances in Neural Information Processing Systems},
  volume={30},
  year={2017}
}

@inproceedings{pustejovsky2025fpo,
  title     = {Frictive Policy Optimization for LLM Agent Interactions},
  author    = {Pustejovsky, James and Krishnaswamy, Nikhil},
  booktitle = {Proceedings of the 24th International Conference on Autonomous Agents and Multiagent Systems (AAMAS)},
  year      = {2025},
  address   = {Detroit, Michigan, USA}
}

@article{wen2025know,
  title={Know your limits: A survey of abstention in large language models},
  author={Wen, Bingbing and Yao, Jihan and Feng, Shangbin and Xu, Chenjun and Tsvetkov, Yulia and Howe, Bill and Wang, Lucy Lu},
  journal={Transactions of the Association for Computational Linguistics},
  volume={13},
  pages={529--556},
  year={2025},
  publisher={MIT Press 255 Main Street, 9th Floor, Cambridge, Massachusetts 02142, USA~…}
}

@inproceedings{umair2024large,
  title={Large language models know what to say but not when to speak},
  author={Umair, Muhammad and Sarathy, Vasanth and Ruiter, Jan},
  booktitle={Findings of the Association for Computational Linguistics: EMNLP 2024},
  pages={15503--15514},
  year={2024}
}

@article{suri2025structured,
  title={Structured Uncertainty guided Clarification for LLM Agents},
  author={Suri, Manan and Mathur, Puneet and Lipka, Nedim and Dernoncourt, Franck and Rossi, Ryan A and Manocha, Dinesh},
  journal={arXiv preprint arXiv:2511.08798},
  year={2025}
}

@inproceedings{ng2000algorithms,
  title={Algorithms for Inverse Reinforcement Learning},
  author={Ng, Andrew Y and Russell, Stuart J},
  booktitle={Proceedings of the Seventeenth International Conference on Machine Learning},
  pages={663--670},
  year={2000}
}

@inproceedings{wulfmeier2015maximum,
  title     = {Maximum Entropy Deep Inverse Reinforcement Learning},
  author    = {Wulfmeier, Markus and Ondr{\'u}{\v{s}}ka, Peter and Posner, Ingmar},
  booktitle = {Proceedings of the 24th International Joint Conference on Artificial Intelligence (IJCAI)},
  pages     = {1561--1567},
  year      = {2015}
}

@inproceedings{khebour2024common,
  title={Common ground tracking in multimodal dialogue},
  author={Khebour, Ibrahim Khalil and Lai, Kenneth and Bradford, Mariah and Zhu, Yifan and Brutti, Richard A and Tam, Christopher and Tu, Jingxuan and Ibarra, Benjamin A and Blanchard, Nathaniel and Krishnaswamy, Nikhil and Pustejovsky, James},
  booktitle={Proceedings of the 2024 Joint International Conference on Computational Linguistics, Language Resources and Evaluation (LREC-COLING 2024)},
  pages={3587--3602},
  year={2024}
}

@inproceedings{jiang2023generative,
  title={Generative Calibration for In-context Learning},
  author={Jiang, Zhongtao and Zhang, Yuanzhe and Liu, Cao and Zhao, Jun and Liu, Kang},
  booktitle={The 2023 Conference on Empirical Methods in Natural Language Processing},
    year={2023}
}

@article{obiso2025dynamic,
  title={Dynamic Epistemic Friction in Dialogue},
  author={Obiso, Timothy and Lai, Kenneth and Nath, Abhijnan and Krishnaswamy, Nikhil and Pustejovsky, James},
  journal={arXiv preprint arXiv:2506.10934},
  year={2025}
}

@inproceedings{nath2025learning,
  title={Learning “Partner-Aware” Collaborators in Multi-Party Collaboration},
  author={Nath, Abhijnan and Krishnaswamy, Nikhil},
  booktitle={The Thirty-ninth Annual Conference on Neural Information Processing Systems},
  year={2025}
}

@inproceedings{nath2026collaborate,
    title={Collaborate, Deliberate, Evaluate: How {LLM} Alignment Affects Coordinated Multi-Agent Outcomes},
    author={Abhijnan Nath and Carine Graff and Nikhil Krishnaswamy},
    booktitle={The 25th International Conference on Autonomous Agents and Multi-Agent Systems},
    year={2026},
    url={https://openreview.net/forum?id=Nph91xmFhl}
}

@article{bai2022constitutional,
  title={Constitutional ai: Harmlessness from ai feedback},
  author={Bai, Yuntao and Kadavath, Saurav and Kundu, Sandipan and Askell, Amanda and Kernion, Jackson and Jones, Andy and Chen, Anna and Goldie, Anna and Mirhoseini, Azalia and McKinnon, Cameron and others},
  journal={arXiv preprint arXiv:2212.08073},
  year={2022}
}

@article{clark2025epistemic,
  title={Epistemic Alignment: A Mediating Framework for User-LLM Knowledge Delivery},
  author={Clark, Nicholas and Shen, Hua and Howe, Bill and Mitra, Tanushree},
  journal={arXiv preprint arXiv:2504.01205},
  year={2025}
}

@article{young2013pomdp,
  title={POMDP-based statistical spoken dialog systems: A review},
  author={Young, Steve and Gasic, Milica and Thomson, Blaise and Williams, Jason},
  journal={Proceedings of the IEEE},
  volume={101},
  number={5},
  pages={1160--1179},
  year={2013}
}

@inproceedings{hou2024input,
  title={Decomposing uncertainty for large language models through input clarification ensembling},
  author={Hou, Bairu and Liu, Yujian and Qian, Kaizhi and Andreas, Jacob and Chang, Shiyu and Zhang, Yang},
  booktitle={Proceedings of the 41st International Conference on Machine Learning},
  pages={19023--19042},
  year={2024}
}

@article{Wang2024ReductionsRSRL,
  title={A reductions approach to risk-sensitive reinforcement learning with optimized certainty equivalents},
  author={Wang, Kaiwen and Liang, Dawen and Kallus, Nathan and Sun, Wen},
  journal={arXiv preprint arXiv:2403.06323},
  year={2024}
}

@book{puterman1994markov,
  title={Markov Decision Processes: Discrete Stochastic Dynamic Programming},
  author={Puterman, Martin L},
  publisher={Wiley},
  year={1994}
}

@article{howard1972risk,
  title={Risk-sensitive Markov decision processes},
  author={Howard, Ronald A and Matheson, James E},
  journal={Management Science},
  volume={18},
  number={7},
  pages={356--369},
  year={1972}
}

@inproceedings{chow2015risk,
  title={Risk-sensitive and robust decision-making: a CVaR optimization approach},
  author={Chow, Yinlam and Ghavamzadeh, Mohammad and Janson, Lucas and Pavone, Marco},
  booktitle={Proceedings of NeurIPS},
  year={2015}
}

@inproceedings{tamar2015policy,
  title={Policy gradient for coherent risk measures},
  author={Tamar, Aviv and Chow, Yinlam and Ghavamzadeh, Mohammad and Mannor, Shie},
  booktitle={Advances in Neural Information Processing Systems},
  volume={28},
  pages={1468--1476},
  year={2015}
}

@inproceedings{achiam2017constrained,
  title={Constrained policy optimization},
  author={Achiam, Joshua and Held, David and Tamar, Aviv and Abbeel, Pieter},
  booktitle={Proceedings of ICML},
  year={2017}
}

@article{sutton1999between,
  title={Between MDPs and Semi-MDPs: A Framework for Temporal Abstraction in Reinforcement Learning},
  author={Sutton, Richard S and Precup, Doina and Singh, Satinder},
  journal={Artificial Intelligence},
  volume={112},
  number={1--2},
  pages={181--211},
  year={1999}
}

@inproceedings{bacon2017option,
  title={The Option-Critic Architecture},
  author={Bacon, Pierre-Luc and Harb, Jean and Precup, Doina},
  booktitle={Proceedings of AAAI},
  year={2017}
}

@article{williams2007partially,
  title={Partially Observable Markov Decision Processes for Spoken Dialogue Management},
  author={Williams, Jason D and Young, Steve},
  journal={Computer Speech \& Language},
  volume={21},
  number={2},
  pages={393--422},
  year={2007}
}

@book{sutton2018reinforcement,
  title={Reinforcement Learning: An Introduction},
  author={Sutton, Richard S and Barto, Andrew G},
  edition={2},
  publisher={MIT Press},
  year={2018}
}

@article{ng1999policy,
  title={Policy invariance under reward transformations: Theory and application to reward shaping},
  author={Ng, Andrew Y and Harada, Daishi and Russell, Stuart},
  journal={ICML},
  year={1999}
}

@inproceedings{niculescu2005predicting,
  title={Predicting good probabilities with supervised learning},
  author={Niculescu-Mizil, Alexandru and Caruana, Rich},
  booktitle={ICML},
  year={2005}
}

@inproceedings{guo2017calibration,
  title={On Calibration of Modern Neural Networks},
  author={Guo, Chuan and Pleiss, Geoff and Sun, Yu and Weinberger, Kilian},
  booktitle={ICML},
  year={2017}
}

@inproceedings{bowman2015snli,
  title={A large annotated corpus for learning natural language inference},
  author={Bowman, Samuel R and Angeli, Gabor and Potts, Christopher and Manning, Christopher D},
  booktitle={EMNLP},
  year={2015}
}

@inproceedings{williams2018mnli,
  title={A Broad-Coverage Challenge Corpus for Sentence Understanding through Inference},
  author={Williams, Adina and Nangia, Nikita and Bowman, Samuel},
  booktitle={NAACL},
  year={2018}
}

@article{weidinger2021ethical,
  title={Ethical and social risks of harm from language models},
  author={Weidinger, Laura and Mellor, John and Rauh, Maribeth and Griffin, Conor and Uesato, Jonathan and Huang, Po-Sen and Cheng, Myra and Glaese, Mia and Balle, Borja and Kasirzadeh, Atoosa and others},
  journal={arXiv preprint arXiv:2112.04359},
  year={2021}
}

@book{bertsekas1995dynamic,
  title={Dynamic Programming and Optimal Control},
  author={Bertsekas, Dimitri P},
  publisher={Athena Scientific},
  year={1995}
}

@article{traum2003information,
  title={Information state update for spoken dialogue systems},
  author={Traum, David and Larsson, Staffan},
  journal={SIGDIAL},
  year={2003}
}

@article{lindley1956measure,
  title={On a measure of the information provided by an experiment},
  author={Lindley, Dennis V},
  journal={Annals of Mathematical Statistics},
  year={1956}
}

@inproceedings{gervits2021decision,
  title={Decision-theoretic question generation for situated reference resolution: An empirical study and computational model},
  author={Gervits, Felix and Briggs, Gordon and Roque, Antonio and Kadomatsu, Genki A and Thurston, Dean and Scheutz, Matthias and Marge, Matthew},
  booktitle={Proceedings of the 2021 international conference on multimodal interaction},
  pages={150--158},
  year={2021}
}

@article{williams1992simple,
  title   = {Simple Statistical Gradient-Following Algorithms for Connectionist Reinforcement Learning},
  author  = {Williams, Ronald J.},
  journal = {Machine Learning},
  volume  = {8},
  number  = {3},
  pages   = {229--256},
  year    = {1992}
}

@inproceedings{sutton2000policy,
  title     = {Policy Gradient Methods for Reinforcement Learning with Function Approximation},
  author    = {Sutton, Richard S. and McAllester, David and Singh, Satinder and Mansour, Yishay},
  booktitle = {Advances in Neural Information Processing Systems (NIPS)},
  year      = {2000}
}

@article{astrom1965optimal,
  title={Optimal control of Markov decision processes with incomplete state information},
  author={{\AA}str{\"o}m, Karl Johan},
  journal={Journal of Mathematical Analysis and Applications},
  volume={10},
  number={1},
  pages={174--205},
  year={1965}
}

@article{smallwood1973optimal,
  title={The optimal control of partially observable Markov processes over a finite horizon},
  author={Smallwood, Richard D and Sondik, Edward J},
  journal={Operations Research},
  volume={21},
  number={5},
  pages={1071--1088},
  year={1973}
}

@article{kaelbling1998planning,
  title={Planning and acting in partially observable stochastic domains},
  author={Kaelbling, Leslie Pack and Littman, Michael L and Cassandra, Anthony R},
  journal={Artificial Intelligence},
  volume={101},
  number={1--2},
  pages={99--134},
  year={1998}
}

@phdthesis{traum1994computational,
  title={A computational theory of grounding in natural language conversation},
  author={Traum, David R},
  school={University of Rochester},
  year={1994}
}

@inproceedings{henderson2014word,
  title={Word-based dialog state tracking with recurrent neural networks},
  author={Henderson, Matthew and Thomson, Blaise and Williams, Jason},
  booktitle={Proceedings of SIGDIAL},
  pages={292--299},
  year={2014}
}

@inproceedings{lee2019sumbt,
  title={SUMBT: Slot-utterance matching for universal and scalable belief tracking},
  author={Lee, Hwaran and Lee, Jinsik and Kim, Tae-Yoon},
  booktitle={Proceedings of ACL},
  year={2019}
}

@article{rafailov2023direct,
  title={Direct preference optimization: Your language model is secretly a reward model},
  author={Rafailov, Rafael and Sharma, Archit and Mitchell, Eric and Manning, Christopher D. and Finn, Chelsea},
  journal={Advances in Neural Information Processing Systems},
  volume={36},
  year={2023}
}

@book{clark1996using,
  title={Using Language},
  author={Clark, Herbert H.},
  publisher={Cambridge University Press},
  year={1996}
}

@article{pust-zhu2026,
  author    = {James Pustejovsky and Yifan Zhu},
   title={Typed Frictive Interventions: A Qualia-Structured Account of Epistemic Control Actions in Dialogue},
 journal = {Unpublished Manuscript, Brandeis University},
  year      = {2026}
}

@article{munn2024truth,
  title={Truth machines: synthesizing veracity in AI language models},
  author={Munn, Luke and Magee, Liam and Arora, Vanicka},
  journal={AI \& society},
  volume={39},
  number={6},
  pages={2759--2773},
  year={2024},
  publisher={Springer}
}

@article{devilling2025polite,
  title={The Polite Liar: Epistemic Pathology in Language Models},
  author={DeVilling, Bentley},
  journal={arXiv preprint arXiv:2511.07477},
  year={2025}
}

@article{ouyang2022training,
  title={Training language models to follow instructions with human feedback},
  author={Ouyang, Long and Wu, Jeffrey and Jiang, Xu and Almeida, Diogo and Wainwright, Carroll and Mishkin, Pamela and Zhang, Chong and Agarwal, Sandhini and Slama, Katarina and Ray, Alex and others},
  journal={Advances in neural information processing systems},
  volume={35},
  pages={27730--27744},
  year={2022}
}

@article{mnih2016asynchronous,
  title={Asynchronous Methods for Deep Reinforcement Learning},
  author={Mnih, Volodymyr et al.},
  journal={ICML},
  year={2016}
}

@article{rubinstein1999cross,
  title={The Cross-Entropy Method for Combinatorial and Continuous Optimization},
  author={Rubinstein, Reuven Y.},
  journal={Methodology and Computing in Applied Probability},
  year={1999}
}

@article{salimans2017evolution,
  title={Evolution Strategies as a Scalable Alternative to Reinforcement Learning},
  author={Salimans, Tim et al.},
  journal={arXiv preprint arXiv:1703.03864},
  year={2017}
}

@article{schulman2015trust,
  title={Trust Region Policy Optimization},
  author={Schulman, John and Levine, Sergey and Moritz, Philipp and Jordan, Michael and Abbeel, Pieter},
  journal={ICML},
  year={2015}
}

@article{schulman2017proximal,
  title={Proximal Policy Optimization Algorithms},
  author={Schulman, John and Wolski, Filip and Dhariwal, Prafulla and Radford, Alec and Klimov, Oleg},
  journal={arXiv preprint arXiv:1707.06347},
  year={2017}
}

@article{haarnoja2018soft,
  title={Soft Actor-Critic: Off-Policy Maximum Entropy Deep Reinforcement Learning with a Stochastic Actor},
  author={Haarnoja, Tuomas and Zhou, Aurick and Abbeel, Pieter and Levine, Sergey},
  journal={ICML},
  year={2018}
}

@article{garcia2015comprehensive,
  title={A Comprehensive Survey on Safe Reinforcement Learning},
  author={Garc{\'\i}a, Javier and Fern{\'a}ndez, Fernando},
  journal={Journal of Machine Learning Research},
  volume={16},
  pages={1437--1480},
  year={2015}
}

@inproceedings{nath2025frictional,
    title = "Frictional Agent Alignment Framework: Slow Down and Don{'}t Break Things",
    author = "Nath, Abhijnan  and
      Graff, Carine  and
      Bachinin, Andrei  and
      Krishnaswamy, Nikhil",
    editor = "Che, Wanxiang  and
      Nabende, Joyce  and
      Shutova, Ekaterina  and
      Pilehvar, Mohammad Taher",
    booktitle = "Proceedings of the 63rd Annual Meeting of the Association for Computational Linguistics (Volume 1: Long Papers)",
    month = jul,
    year = "2025",
    address = "Vienna, Austria",
    publisher = "Association for Computational Linguistics",
    url = "https://aclanthology.org/2025.acl-long.542/",
    doi = "10.18653/v1/2025.acl-long.542",
    pages = "11042--11089",
    ISBN = "979-8-89176-251-0"
}

\end{document}